\documentclass[10pt,reqno]{amsart}
\usepackage{companypaper}
\usepackage{algorithm}
\usepackage{algorithmic}
\usepackage{natbib}
\usepackage{hyphenat}
\usepackage{enumitem}
\usepackage{verbatim}

\makeatletter
\pretocmd{\@settitle}{\let\uppercasenonmath\@gobble}{}{}
\patchcmd{\@settitle}{\bfseries}{\bfseries\LARGE}{}{}
\pretocmd{\@setauthors}{\let\MakeUppercase\@firstofone}{}{}
\patchcmd{\@setauthors}{\centering\footnotesize}{\centering\large}{}{}
\apptocmd{\@setauthors}{}{}{}
\renewcommand{\@setaddresses}{}
\makeatother

\usepackage{pifont}
\usepackage{booktabs}
\usepackage{multirow}
\usepackage{threeparttable}
\usepackage{array}
\usepackage{wrapfig}
\usepackage{tabularx}
\newtheorem{Corollary}{Corollary}

\usepackage{fontawesome5}

\newtcolorbox{codebox}[1][]{
  sharp corners,
  colback=PyCodeBg,
  colframe=PyCodeBg,
  boxrule=0pt,
  left=0mm, right=0mm, top=1mm, bottom=1mm,
  #1
}

\usepackage{newfloat}
\DeclareFloatingEnvironment[
  fileext=loc,
  listname={List of Codes},
  name=Code,
  placement=t,
]{codefloat}

\usepackage{nicefrac}
\usepackage{physics2}
\usepackage{fixdif, derivative}
\usephysicsmodule{ab}

\CompanyLogo{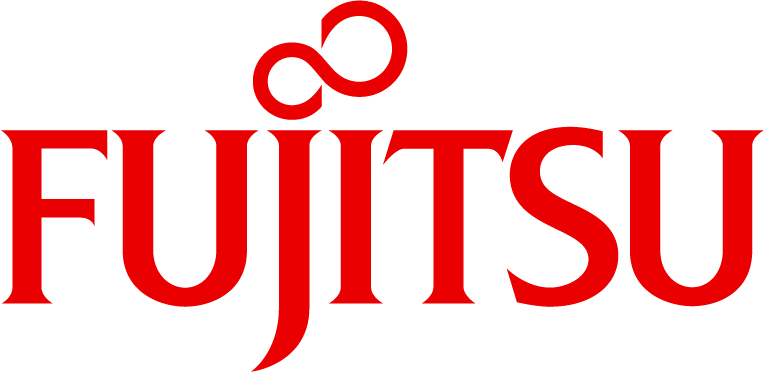}
\CompanyLogoHeight{11mm}
\CompanyHeaderText{Fujitsu Limited}
\CompanyDocType{}
\CompanyClassification{}
\CompanySubtitle{}

\title{HoT-SSM:Higher-order Temporal Knowledge Graph Reasoning with State Space Models for Health Care}

\author{%
  Thummaluru Siddartha Reddy$^{*}$ \quad
  Vempalli Naga Sai Saketh$^{*}$ \quad
  Yash Punjabi$^{*}$ \\[4pt]
  Mahesh Chandran \\[4pt]
  Fujitsu Research of India, Bangalore \\[2pt]
  \texttt{\{Thummaluru.Siddarthareddy, nagasaisaketh.vempalli\}@fujitsu.com} \\
  \texttt{\{yash.punjabi, mahesh.chandran\}@fujitsu.com} \\
}

\begin{document}

\maketitle
\renewcommand{\thefootnote}{\fnsymbol{footnote}}
\footnotetext{$^*$Equal Contribution.}
\begin{center}
\small

\end{center}

\begin{abstract}
  Medical knowledge graphs (MKGs) infused with clinical knowledge have been increasingly used to model electronic health records (EHRs) to support interpretable predictions in healthcare domain. However, existing MKG-based approaches are limited in capturing pairwise relations between clinical concepts (e.g., conditions, procedures, and medications), and restricts their ability to model higher-order interactions among co-occurring or semantically related concepts. In addition, most representation learning methods that leverage MKGs either collapse temporal information across visits or lack an explicit mechanism for modeling long-range temporal dependencies, which is critical for  clinical tasks such as mortality prediction. To mitigate these limitations, we propose \texttt{HoT-SSM}, a parameter efficient and higher-order temporal graph reasoning with state space models. For each visit, \texttt{HoT-SSM} constructs hypergraphs by grouping semantically related clinical concepts into hyperedges using domain knowledge, thereby preserving visit-level clinical context. Further, to model the temporal dynamics while learning the representations, we introduce a novel dynamic hypergraph-based state space model that explicitly captures patients latent state evolution over time while preserving long-range information. The learned representations are used for downstream clinical prediction and reasoning. Experiments on MIMIC-III and MIMIC-IV datasets  shows significant performance improvement over the current state-of-the-art models, demonstrating the effectiveness of jointly modeling higher-order clinical interactions and long-range temporal dependencies.
\end{abstract}

\section{Introduction}
\label{submission}

Knowledge graphs (KGs) offer a structured and relational representation of domain-specific knowledge, encoding the semantic relationships between the entities \citep{hogan2021knowledge,gao2025leveraging}. Specifically, KGs augmented with large language models (LLMs) as an inductive bias are increasingly becoming popular in enhancing the reasoning and generalization capabilities of LLMs, particularly in domains where factual consistency and interpretability are critical \citep{wu2024medical,wu2025medical,jiang2023think}. In the medical domain, medical knowledge graphs (MKGs) encode clinical ontologies that formalize medical concepts-such as diagnoses, procedures, and medications-and their clinical relationships \citep{aldughayfiq2023capturing, shirai2021applying}. Integrating MKGs with electronic health record (EHR) data has demonstrated improved performance and interpretability on downstream clinical tasks such as mortality prediction, drug recommendation,  to name a few \citep{KARE,Graphcare,gao2025leveraging,jiang2023think}. 

Applications of artificial intelligence in healthcare have evolved from traditional machine learning based models for clinical decision support \citep{choi2016retain, choi2016doctor, choi2017gram, zhang2021grasp} to more recent approaches that employ large language models (LLMs) to perform complex clinical reasoning \citep{jiang2023think, wu2024medical, wu2025medical}. Existing methods for modeling EHR data can be broadly classified into two categories. The first category focuses on representation learning, where deep neural networks or graph neural networks are used to learn latent patient-level or concept-level embeddings for downstream prediction tasks \citep{ma2018kame, ma2020concare, zhang2021grasp, Graphcare}. The second category are LLM driven approaches, in which LLM's pretrained on general and biomedical corpus are integrated with medical knowledge graphs (MKGs). In these methods, structured medical concepts and relations from MKGs are used to ground and constrain the reasoning process of LLMs, enabling more interpretable and clinically meaningful inference \citep{KARE, Graphcare, gao2025leveraging}. While representation-based models excel at capturing global contextual patterns in EHR data, LLM-based approaches offer enhanced interpretability through explicit reasoning. These complementary strengths have motivated recent efforts to integrate structured knowledge representations with neural models for clinical prediction.

Approaches such as \texttt{KARE} \citep{KARE} and \texttt{GraphCare} \citep{Graphcare} leverage domain knowledge to build personalized KGs and subsequently perform inference using either LLM-based reasoning or graph neural networks. While effective in incorporating external medical knowledge, these methods exhibit important limitations. In particular, both \texttt{KARE} and \texttt{GraphCare} restrict interactions between medical concepts to pairwise relations, which fundamentally limits their ability to model higher-order clinical patterns. In real-world EHR data, clinical outcomes are often driven by the joint co-occurrence of multiple related concepts within a visit. For example, severe respiratory failure is characterized by the simultaneous presence of acute respiratory distress syndrome, hypoxemia, and mechanical ventilation; modeling such concepts using pairwise relations fails to capture their collective severity which is critical for tasks such as mortality prediction. Furthermore, existing MKG-based approaches either collapse temporal information across visits or lack an explicit mechanism for modeling latent temporal state evolution, thereby limiting their ability to capture long-range disease progression and patient history. In addition, they  incur significant computational costs to obtain the predictions (see Fig.~\ref{fig:pred_size}).  These challenges motivate the need for a unified parameter efficient framework that can simultaneously represent higher-order clinical interactions and explicitly model temporal dynamics  while preserving the information for long range.

\begin{wrapfigure}{r}{0.42\textwidth}  
  \centering
\includegraphics[width=0.40\textwidth]{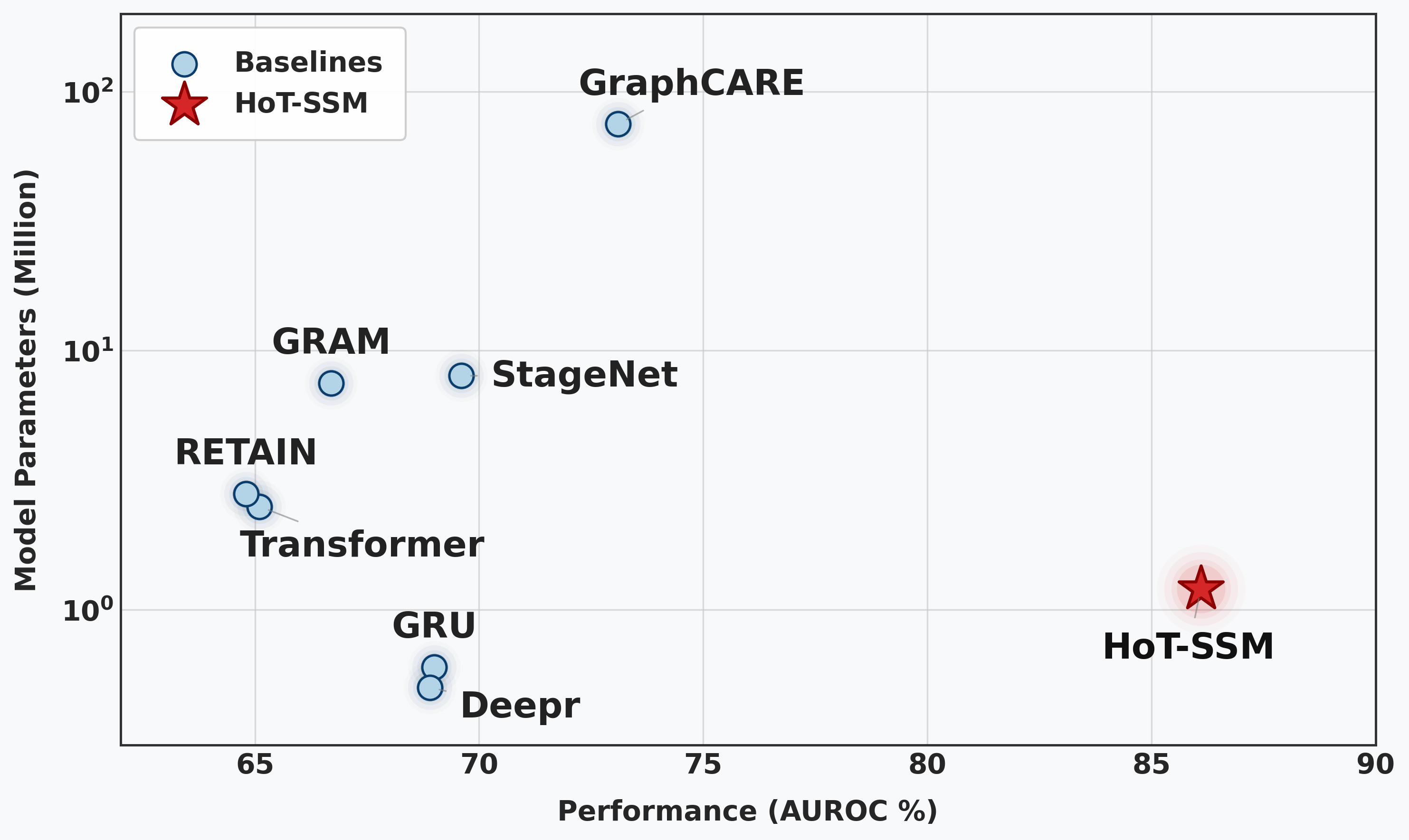}
  \caption{Performance vs model size on MIMIC-IV (mortality).}
  
  \label{fig:pred_size}
\end{wrapfigure}

In this work, we introduce \texttt{HoT-SSM}, a new parameter-efficient framework for EHR modeling that preserves long-range temporal information across clinical visits while capturing higher-order relationships among medical concepts through knowledge infused hypergraphs.
Specifically, \texttt{HoT-SSM} represents patient-specific EHR data as a temporal sequence of knowledge infused hypergraphs pertaining to each clinical visit; the knowledge infusion is achieved by connecting semantically related conditions, medications, and procedures by hyperedges, extracted from a global hyperknowledge graph constructed over the dataset. Resulting temporal hypergraphs are processed by the proposed dynamic state space model (SSM) to explicitly learn the evolving latent state of the patient which simultaneously encodes higher-order relational structure (due to hypergraphs) and  preserves long-range temporal information (due to SSM). Empirical results on MIMIC-III and MIMIC-IV across multiple clinical tasks shows markedly improved performance, demonstrating  its the effectiveness.
We summarize the main contributions as follows:
\begin{itemize}
\item A Temporal Knowledge infused hypergraphs is introduced to model patient-specific EHR data, explicitly capturing higher-order clinical relationships and their temporal evolution across visits.
\item A dynamic hypergraph state space model that jointly preserves higher-order spatial interactions and long-range temporal information is proposed to model the temporal evolution of the patient state. Furthermore, we establish theoretical guarantees for the robustness of the representations from \texttt{HoT-SSM} to perturbations and permutation of graph structure.
\item A temporal reasoning mechanism is proposed to extract structured reasoning paths from the learned latent state representations, which are subsequently verbalized using a lightweight language model to produce human-interpretable explanations.
\end{itemize}

\section{Preliminaries}
In this section, we setup the mathematical background required for the proposed framework.
\subsection{Knowledge Graphs and Temporal Knowledge Graphs}
Let $\mathcal{V}$ denote a finite set of entities (subjects and objects), and $\mathcal{R}$ denote a set of relations. A \emph{KG} encodes the structured relationships and is represented  as
$\mathcal{G} = (\mathcal{V}, \mathcal{R}, \mathcal{E})$, where $\mathcal{E} \subseteq \mathcal{V} \times \mathcal{R} \times \mathcal{V}$ is a set of relational triples. In particular, each triple
$
e = (s, r, o) \in \mathcal{E}
$
represents a relationship of type $ r \in \mathcal{R}$ between a subject entity $s \in \mathcal{V}$ and an object entity $o \in \mathcal{V}$.

\emph{Temporal knowledge graphs} (TKGs) are a sequence of time-indexed KGs
$
\mathcal{T_G} = \{ \mathcal{G}_1, \mathcal{G}_2, \ldots, \mathcal{G}_T\}
$ with  each $\mathcal{G}_{t} \in \mathcal{T_{G}}$ and relations are timestamped. In particular, each temporal triple
$
e_t = (s, r, o) \in \mathcal{E}_t
$
encodes the relationship between the entities at time $t$.


\subsection{Hypergraphs}
Hypergraphs are higher order abstractions of graphs where the hyperedges encodes relationship between multiple entities $\geq2$ \citep{yadati2019hypergcn,feng2019hypergraph}. Formally, a hypergraph is defined as $ \mathcal{H} = (\mathcal{V}, \mathcal{E})$,
where  $\mathcal{E}$ denotes the set of hyperedges, with each hyperedge $ e_{h} \in \mathcal{E} $ connects a subset of entities such that  $|e_{h}| \geq 2$. The relationship between entities and hyperedges is encoded by the incidence matrix
$
\mathbf{I} \in \mathbb{R}^{|\mathcal{V}| \times |\mathcal{E}|},
$
and has non-zero entries if an entity participates in a hyperedge.  
\subsection{Temporal Graph State Space Models} 
State Space Models (SSMs) are computationally efficient temporal representation learning architectures that are well suited for modeling long-range temporal dependencies \citep{gu2024mamba}.  Temporal Graph State Space Models \texttt{(T-SSM)} extends classical SSMs to dynamic graph-structured data \citep{li2024state}. At each time step $k$, \texttt{T-SSM} updates the latent node representations $\mathbf{S}_k$ as
\begin{equation}
\begin{aligned}
\mathbf{S}_{k} &= \mathbf{S}_{k-1} \exp(\Delta_{k}\mathbf{A}) + 
\tilde{\mathbf{X}}_{k} \big(\exp(\Delta_{k}\mathbf{A}) - \mathbf{I}\big)\mathbf{A}^{-1},
\label{eq:ssm}
\nonumber
  \end{aligned}
  \end{equation}
Here $\tilde{\mathbf{X}}_{k}$  encodes the spatial structure of the graph and is obtained  by $ \tilde{\mathbf{X}}_{k} = \mathrm{GNN}_{\boldsymbol{\Theta}}(\mathbf{L}_{k}, \mathbf{X}_{k}) \mathbf{B}^{T}
$. Whereas $\mathbf{L}_{k}$ and $\mathbf{X}_{k}$ denotes the graph Laplacian and  feature matrix at time step k and $\Delta_{k}$ is the step size. The matrices $\mathbf{A}$ and $\mathbf{B}$ are state and input parameters, respectively.

\section{Knowledge Infused Hypergraph Reasoning with SSM}

In this section, we introduce the proposed \texttt{HoT-SSM} framework for modeling EHR data. We first present a knowledge infused temporal hypergraph representation, where domain knowledge guides the hyperedge construction while patient-specific graphs are instantiated from observed clinical data. We then introduce a dynamic hypergraph-based state space model (SSM) to capture temporal dependencies, followed by methods for extracting interpretable temporal reasoning paths. Figure~\ref{fig:architecture} illustrates the overall architecture.

\subsection{Knowledge Infused Hypergraph Representation of EHR Data} \label{sec:hypergraph_generation}
Traditional knowledge graph that captures the pairwise relationship between the entities, does not exploit the full complexities of the multi-way relationships among the clinical concepts, including conditions, procedures, and medications. To explicitly include these higher-order relations that is intrinsic to EHR data, we introduce a temporal hypergraph-based knowledge representation of it.

Let $\mathcal{P} = \{P_1, P_2, \ldots, P_N\}$ denote a set of $N$ patients with each patient $P_i$ being associated with a sequence of clinical visits over time. In particular, considering that patient $P_i$ has $T_i$ visits, indexed by $t \in \{1, 2, \ldots, T_i\}$. The clinical data corresponding to patient $P_i$ at visit $t$ is denoted by
$
C_{i,t} = \{ \mathcal{C}_{i,t}, \mathcal{D}_{i,t}, \mathcal{R}_{i,t} \},
$
where $\mathcal{C}_{i,t}$, $\mathcal{D}_{i,t}$, and $\mathcal{R}_{i,t}$ represent conditions, drugs, and procedures.\\

\begin{figure*}[t]
    \centering
    \includegraphics[width=\textwidth]{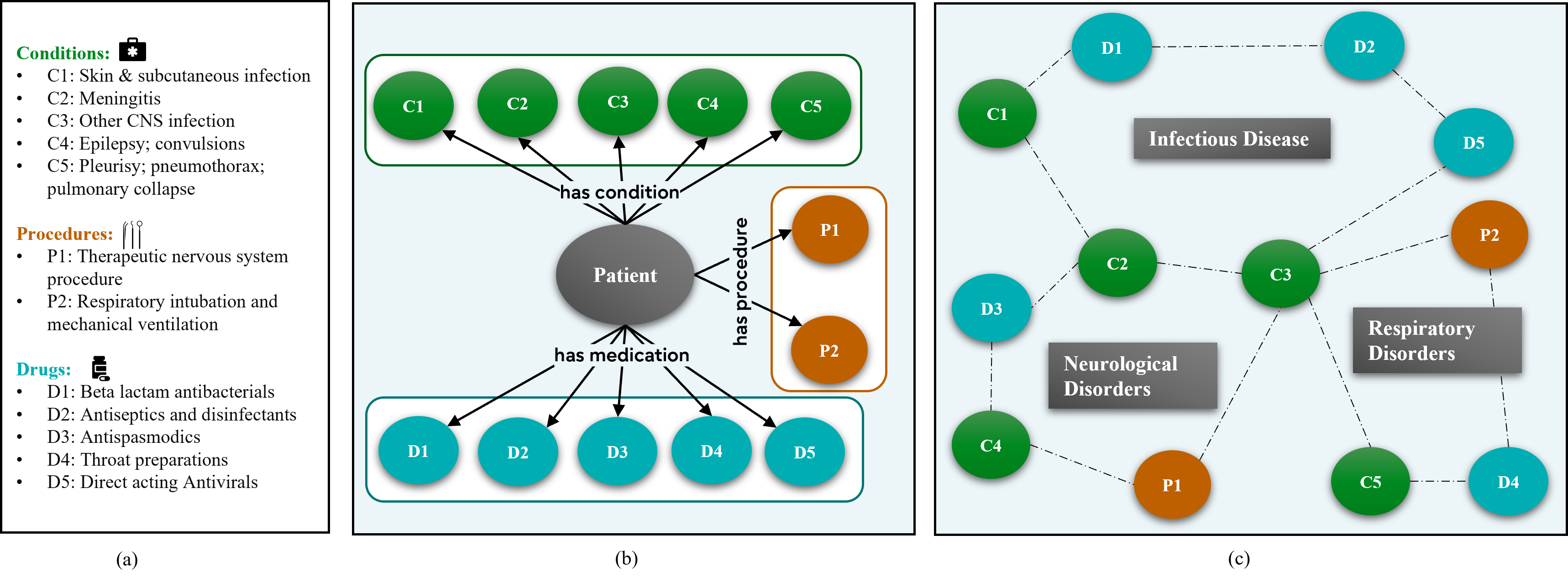}
    \caption{(a) Sample from EHR data. (b) Pairwise knowledge graph. (c) Knowledge hypergraph.}
\label{fig:hypergraph_motivation}
\end{figure*}

\textbf{Knowledge Infused Hypergraph Construction}: For each visit $t \in T_i$ of patient $P_i$, we construct a visit-level hypergraph to model higher-order clinical interactions through hyperedges that explicitly captures co-occurring conditions, medications, and procedures within a single visit. A naive approach would invoke a large language model (LLM) to construct a hypergraph independently for each visit. However, such a strategy incurs prohibitive token and computational costs. To address this limitation, we propose a cost-efficient two-stage  procedure:

\textit{Stage 1}: We first construct a \emph{global} hyperknowledge graph (HKG) over the EHR concept vocabulary by aggregating unique clinical concepts across all patients. Using LLM prompting, we group semantically related conditions, medications, and procedures into clinically coherent concept sets. Each group defines a hyperedge, capturing meaningful higher-order clinical context. The HKG can be expanded by pooling information from external sources such as PubMed and UMLS along with the derived medical ontologies.

\textit{Stage 2}: In the second step, the knowledge from global HKG is transferred to patient-specific hypergraph, instantiated for each visit, by selecting the subset of global hyperedges whose concepts are present in that visit. This approach of \emph{infusing} knowledge into the hypergraph avoids repeated LLM calls, thus optimizing the compute and cost that accompanies use of LLM. To explicitly encode visit-level context, we introduce a dedicated \emph{visit node} connected to all hyperedges in the corresponding hypergraph. Further details on the prompting strategy and construction are provided in Appendix~\ref{appendix: hyperknowledge_graph}.

Formally, for a visit $t$ of patient $P_i$, we define a hypergraph as $
\mathcal{H}_{i,t} = ({\mathcal{V}}_{i,t}, {\mathcal{E}}_{i,t})$,
where $\mathcal{V}_{i,t}$ denotes the set of unique clinical entities appearing in $C_{i,t}$  and visit node,  ${\mathcal{E}}_{i,t} = \{e_{i,t}^{(1)}, e_{i,t}^{(2)}, \ldots, e_{i,t}^{(K_{i,t})}\}$ is the set of hyperedges inferred by the LLM. A single visit can include multiple hyperedges, indicating that patient has multiple dissimilar concepts ($K_{i,t}$ for patient $i$). The hyperedge set is therefore a collection of unique concepts at time $t$ and is given by
$
\mathcal{E}_{i,t} = \bigcup_{k=1}^{K_{i,t}} e_{i,t}^{(k)}.
$
The structure of this hypergraph is encoded by an incidence matrix $\mathbf{I}_{i,t} \in \mathbb{R}^{|\mathcal{V}_{i,t}| \times |\mathcal{E}_{i,t}|}$. Node features are initialized using BERT embeddings \citep{devlin2019bertpretrainingdeepbidirectional} and projected to the latent space via a learnable linear layer.

Figure~\ref{fig:hypergraph_motivation}(c) illustrates a knowledge-infused hypergraph constructed from a single visit sample from the MIMIC-III dataset, with clinical concepts shown in Figure~\ref{fig:hypergraph_motivation}(a). In contrast, Figure~\ref{fig:hypergraph_motivation}(b) presents the corresponding pairwise KG, which models conditions, procedures, and medications independently despite their co-occurrence within the same visit. Consequently, it fails to capture the joint clinical context and interdependencies among factors such as neurological, infectious, and respiratory conditions. The knowledge-infused hypergraph addresses this limitation by encoding visit-level information through hyperedges that group clinically related concepts, enabling representation of higher-order disease patterns. For instance, in Figure~\ref{fig:hypergraph_motivation}(c), a hyperedge for respiratory disorders jointly connects pleurisy, respiratory intubation, mechanical ventilation, and associated medications. By modeling such multi-way interactions, the hypergraph provides a more expressive representation that supports richer clinical reasoning and improved downstream performance. \\

\textbf{Knowledge Infused Temporal Hypergraphs.}
For each patient $P_i$, we construct a sequence of visit-level hypergraphs ordered by time i.e., $
\mathcal{H}_i = \{ \mathcal{H}_{i,1}, \mathcal{H}_{i,2}, \ldots, \mathcal{H}_{i,T_i} \}.
$
This enables the proposed model to effectively capture the patient history as sequences of static hypergraphs.  Further, this allows us to apply dynamic representation learning technique to exploit the temporal evolution.

\begin{figure*}[t]
    \centering
    \includegraphics[width=\textwidth]{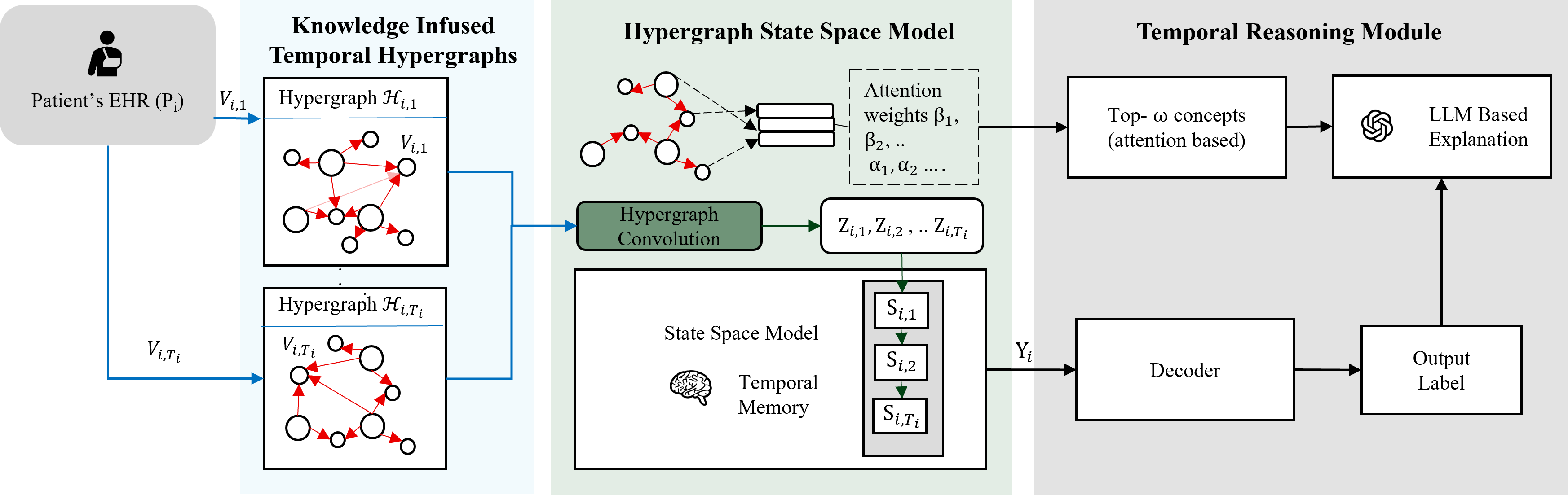}
    \caption{Model workflow, where patient specific temporal hypergraphs are processed through dynamic state space model to learn temporal representations, followed by decoder for prediction and reasoning module for explanation.}
    \label{fig:architecture}
    
\end{figure*}

\subsection{Representation Learning with Hypergraph based Dynamic State Space Model} \label{Sec:hypergraphssm}

In this section, we introduce proposed dynamic state-space model (SSM) over temporal hypergraphs.\\ 
\textbf{Static Hypergraph Embedding.}
We first obtain static representations for each visit leveraging the corresponding hypergraph as an inductive bias for capturing higher-order clinical dependencies.  Let $\mathbf{X}_{i,t} \in \mathbb{R}^{|\mathcal{V}_{i,t}| \times F}$,  denote the feature matrix for patient $i$ at time $t$ with $F$ being feature dimension, and $\mathbf{I}_{i,t} \in \mathbb{R}^{|\mathcal{V}_{i,t}| \times |\mathcal{E}_{i,t}|}$ denote the hypergraph incidence matrix. We compute node embeddings $\mathbf{Z}_{i,t} \in \mathbb{R}^{|\mathcal{V}_{i,t}|\times d}$, using a hypergraph convolution operator \texttt(HConv) \citep{bai2021hypergraph} as
\begin{equation}
\mathbf{Z}^{(l+1)}_{i,t} 
=
\sigma\!\left(
\mathbf{Dv}^{{-1/2}}_{i,t}
\mathbf{I}_{i,t}
\mathbf{W}
\mathbf{De}_{i,t}^{-1}
\mathbf{I}_{i,t}^{\top}
\mathbf{Z}^{(l)}_{i,t}
\boldsymbol{\Theta}^{(l+1)}
\right) ,
\label{eq:hyper_conv}
\end{equation}
where $\mathbf{Dv}_{i,t} \in \mathbb{R}^{|\mathcal{V}_{i,t}| \times |\mathcal{V}_{i,t}|}$ and $\mathbf{De}_{i,t} \in \mathbb{R}^{|\mathcal{E}_{i,t}| \times |\mathcal{E}_{i,t}|}$ denote the node and hyperedge degree matrices, respectively, $\mathbf{W}$ represents learnable hyperedge-specific weights, $\boldsymbol{\Theta}^{(l+1)}$ is the layer-specific trainable parameter matrix, and $\sigma(\cdot)$ is a non-linear activation function.

The update in Eq.~\eqref{eq:hyper_conv} implements a spectral hypergraph convolution, generalizing graph convolution to higher-order relational structures. Message passing is performed via hyperedges, allowing information to be aggregated across groups of related clinical concepts. Further stacking multiple such hypergraph convolution layers, we obtain visit node embeddings that encode rich structural information from the static hypergraphs. These embeddings serve as inputs to temporal state-space model, which captures the temporal evolution of clinical states across visits.
 
\begin{remark}\label{sec:remark}
The \texttt{HConv} operator aggregates information from neighboring nodes using uniform weights and therefore does not provide an explicit mechanism for identifying the most influential clinical concepts. To address this limitation, we introduce an attention-based variant \citep{ding2020more}, to learns weights for neighborhood aggregation and further can be leveraged to extract interpretable reasoning paths underlying the model’s predictions (more details in Sec~\ref{subsec: reasoning_paths}).
\end{remark}

\textbf{Dynamic State-Space Model over Temporal Hypergraphs}. The temporal hypergraphs associated with a patient is naturally represented as a sequence of static knowledge infused hypergraphs evolving over discrete time. We model this sequence as a discrete-time dynamic hypergraphs, enabling principled temporal representation learning over higher-order  structures.  

In particular for EHR data, clinically relevant signals may span long temporal horizons, necessitating framework for capturing long-range temporal information propagation. To address this, we adopt a state-space modeling (SSM) framework, which is well-suited for preserving and updating long-range temporal dependencies. Specifically, for a patient $P_i$ with sequence of hypergraphs as $
\mathcal{H}_i = \{ \mathcal{H}_{i,1}, \mathcal{H}_{i,2}, \ldots, \mathcal{H}_{i,T_i} \}$ the goal is to learn temporal latent representations $\mathbf{s}_{i} \in \mathbb{R}^{1 \times f}$ that summarize the patient’s clinical trajectory and support downstream prediction.

\textbf{HiPPO-based Memory Representation.}
To model long-range temporal dependencies, we employ a Higher-Order Polynomial Projection Operator (HiPPO) framework \citep{gu2020hippo}. HiPPO-based models compress the entire history of inputs into a fixed-dimensional memory by maintaining coefficients of a polynomial basis, which are updated through a state-space formulation.

Let $k\in \{1,2,\ldots T_{i}\}$ index the  discrete time clinical visits of patient $i$, with the $k$th visit at time instant $t_{k}$ and the interval between the visits as $\Delta_{k} = t_{k}-t_{k-1}$. Then the resulting discrete-time state-space update for patient $P_i$ (following \texttt{T-SSM} framework) is given by
\begin{equation}
\mathbf{s}_{i,k}
=
\mathbf{s}_{i,k-1} \exp(\Delta_k \mathbf{A})
+
\tilde{\mathbf{x}}_{i,k}
\left( \exp(\Delta_k \mathbf{A}) - \mathbf{I} \right)\mathbf{A}^{-1},
\label{eq:ssm_equation}
\end{equation}

where $\mathbf{A}\in\mathbb{R}^{f \times f}$,   is a HiPPO matrix and $
\tilde{\mathbf{x}}_{i,k} = \mathrm{HConv}_{v}(\mathbf{X}_{i,k}, \mathcal{H}_{i,k}) \mathbf{B}^{T} \in \mathbb{R}^{1 \times f}$, 
represents the hypergraph-aware embeddings obtained via the hypergraph convolution operator defined in Eq.~\eqref{eq:hyper_conv} with $\mathbf{B} \in \mathbb{R}^{f \times d}$ is a trainable input matrix.  Importantly $\mathrm{HConv}_{v}(\mathbf{X}_{i,k}, \mathcal{H}_{i,k}) \in \mathbb{R}^{1 \times d}$ output the visit node embedding by indexing the visit node in $\mathbf{Z}_{i,k}$.  It can be observed   
from Eq.~\eqref{eq:ssm_equation}, the proposed model jointly captures higher-order spatial dependencies within individual visits and long-range temporal dependencies across visits due to SSM modeling. Thus  the learned temporal representations are highly expressive and well-suited for downstream clinical tasks.

After processing the full visit sequence for patient $P_i$, the final output representation is obtained  as
$\mathbf{y}_i = \mathbf{C}\mathbf{s}_{i,K}^{T}$,
where $\mathbf{s}_{i,K}$ denotes the memory representation after the final visit and $\mathbf{C} \in \mathbb{R}^{d_{1} \times f} $ is a learnable output  matrix. Depending on the downstream task, the representation $\mathbf{y}_i$ is subsequently passed through a task-specific decoder to produce the final prediction.

\subsection{Temporal Reasoning Paths}
\label{subsec: reasoning_paths}
In this section, we discuss the proposed method for extracting temporal reasoning paths. We discuss two variants based on attention and gradient-based attribution.\\
\textbf{Attention based approach}:
To mitigate the challenges from \texttt{HConv} layer as discussed in Remark~\ref{sec:remark}, we introduce an attention based hypergraph convolution (\texttt{HGAT})\citep{bai2021hypergraph} that assigns learnable importance weights to nodes and hyperedges, which can be subsequently leveraged to extract reasoning paths \citep{ding2020more}.

Let $v_{r} \in \mathcal{V}_{i,t}$ denote a vertex $r$ (clinical concept) of patient $i$ at visit $t$. To learn the representation of node $r$, the attention-based hypergraph convolution operates in two stages:
(1) node to hyperedge aggregation, (2) hyperedge to node aggregation. In particular, the representation of node $r$ in layer $l$ i.e., $\mathbf{z}^{(l+1)}_{r} \in \mathbb{R}^{d \times 1}$ is obtained as 
\begin{equation}
\mathbf{z}^{(l+1)}_{r}
= \sigma\left(
\sum_{e^j \in \mathcal{E}_{i,t}(v^r)}
\beta_{r,j}
\mathbf{W}^{(2)} \hat{\mathbf{e}}^{(l+1)}_{j}
\right), \hspace{1cm} \beta_{r,j}
= \frac{
\exp\left(\mathbf{n}^{\top} \tilde{\mathbf{h}}_{i,j}\right)
}{
\sum_{e^q \in \mathcal{E}_{i,t}(v^r)}
\exp\left(\mathbf{n}^{\top} \tilde{\mathbf{h}}_{q,j}\right)
},
\end{equation}
where $\mathcal{E}_{i,t}(v^r)$ denotes set of hyperedges incident to node $v^r$, $e^{j}$ denotes $j^{th}$ hyperedge, $\hat{\mathbf{e}}^{(l+1)}_{j}$ is the updated representation of hyperedge $e^j$ (see Appendix \ref{appendix: attention_based_hykg} for detailed derivation) and $\mathbf{W}^{(2)} \in \mathbb{R}^{d \times d}$ is a learnable weight matrix. $\beta_{r,j}$ is the hyperedge to node attention coefficient that quantifies the influence of hyperedge $e^j$ on node $v^r$. Here 

$\mathbf{n}$ is a learnable vector and

$\tilde{\mathbf{h}}_{i,j}
= \text{LeakyReLU}\left(
\mathbf{W}^{(2)} \hat{\mathbf{e}}_{j}^{(l+1)} ||
\mathbf{W}^{(1)} \mathbf{z}^{(l)}_{r}
\right).$ Learning the coefficients $\beta_{r,j}$ enables the model to explicitly identify the most influential hyperedges contributing to final representation. To extract interpretable reasoning paths, we sort the learned attention coefficients of visit node and select the top-$\omega$ concepts with the highest attention scores using $(\alpha, \, \beta)$. Further, these top-$\omega$ concepts are used to generate the temporal reasoning paths as detailed in the Appendix~\ref{appendix: attention_based_hykg}. \\
\textbf{Gradient based approach}:\label{paragraph: grad_path} 
While the attention-based approach introduced above enables fine-grained identification of important nodes and hyperedges, it requires learning both node- and edge-level attention weights and therefore incurs substantially higher parameter and computational complexity than the \texttt{HConv} operator. In resource constrained settings, we instead adopt a gradient based attribution strategy that directly leverages the embeddings produced by \texttt{HConv}, without introducing additional learnable parameters. Let $\mathbf{x}^{(j)}_{i,k} \in \mathbb{R}^{d}$ denote the embedding of clinical entity $j$ for patient $i$ at visit $k$. To quantify the contribution of this entity to the final prediction $\hat{y}_i$, we compute the gradient of $\hat{y}_i$ with respect to $\mathbf{x}^{(j)}_{i,k}$ i.e, $\nabla_{\mathbf{x}^{(j)}_{i,k}} \hat{y}_{i}$. This gradient reflects how changes in an entity’s representation at a given visit propagate through the latent SSM dynamics and affect the final prediction.

Leveraging the gradients we then compute the saliency score \citep{shrikumar2017learning} which accounts for feature presence and sensitivity. In particular the saliency score $Sc(e, k)$ for entity $j$ of patient $i$ at visit $k$ is given by  $Sc(e, k) = \mathbf{x}^{(j)^{T}}_{i,k} \cdot \nabla_{\mathbf{x}^{(j)}_{i,k}} \hat{y}_{i}$.

A higher positive saliency score indicates that the corresponding clinical entity contributes more strongly toward increasing the outcome.

To translate saliency scores into an interpretable temporal reasoning path, we perform gradient-based backtracking using a set of \emph{anchors}, defined as the top-$\omega$ entities in the final visit $K$ with the highest saliency scores. For each anchor, we reconstruct a reasoning path by tracing backward through time and selecting the entity at visit $k$ i.e., $e_k^{*}$ which maximizes the saliency score $Sc(e,k)$.

\section{Theoretical Guarantees}
\begin{theorem}[Perturbation to hypergraph structures]
\label{thm:perturb}
Let $\hat{\mathbf{L}} = \mathbf{L}+\Delta\mathbf{L}$, be the perturbed hypergraph Laplacian with $\|\Delta\mathbf{L}\|_{2}\leq\epsilon$ and $\sigma(.)$ being Lipschitz continous with constant $C$,  number of layers as $L$ and weight parameter satisfies $\max_{l}||\mathbf{\Theta}^{(l)}||_{2}\leq K$. Further if the state matrices satisfy $\|\mathbf{B}\|_{2} \leq \gamma$ and $\|(exp(\Delta_{k}\mathbf{A})-\mathbf{I})\mathbf{A}^{-1}\| \leq \delta$ then the error between  representations are bounded as 
$\|\hat{\mathbf{s}}_{i,k} - {\mathbf{s}}_{i,k}\|_{2} \leq \rm{const}\cdot \epsilon ||\mathbf{X}_{i,k}||_{2}$, where   $\rm{const} = CKL\gamma\delta$ 
\end{theorem}
The proof of this theorem is relegated to the Appendix \ref{appendix: proof_pertub}. It shows that the error in between the representations learned by \texttt{HoT-SSM} scales linearly with the energy of the perturbation matrix \textit{i.e.}, $\epsilon$ demonstrating the robustness of the proposed model to structural perturbations.
\begin{Corollary}[Error accounting for temporal propagation]\label{sec:cor}
Assume additionally if the discretized state matrix satisfies $\|\rm{exp}(\Delta_{k})\mathbf{A}\|_{2} \leq \eta$, $\forall k$. Then for $K$ steps the error between the state representations accounting for temporal residual propogations are bounded as $ \|\hat{\mathbf{s}}_{i,K}-\mathbf{s}_{i,K}\|_{2} \leq \rm{const_{1}}\epsilon$, where $\rm{const_1}=L K C \delta \gamma \sum_{\tau=1}^{K} \eta^{K-\tau}\|\mathbf{X}_{i,\tau}\|_{2}$. 
\end{Corollary}

    We relegate the proof to Appendix~\ref{sec:proof_cor}. It can be observed that the error remains
linear in the energy of perturbation matrix magnitude $\epsilon$ even while accounting
temporal cascading effects.

\begin{remark}
Theorem~\ref{thm:perturb} and Corollary~\ref{sec:cor} characterize the error in terms of energy of the perturbation energy. Extending this analysis, we derive an alternative bound expressed in terms of the spectrum of the hypergraph Laplacian (see Appendix~\ref{sec:proof_spectrum} for details).
\end{remark}
\begin{theorem}[Permutation invariance]
\label{thm:permutation_invariance}
 Let \(\mathcal{P} = \{\,\mathbf{P}\in\{0,1\}^{|\mathcal{V}_{i,t}|\times |\mathcal{V}_{i,t}|} : \mathbf{P}^\top\mathbf{P} = \mathbf{P}\mathbf{P}^\top = \mathbf{I}_{|\mathcal{V}_{i,t}|}\}\) be the set of all valid  permutation matrices . Under the permutation of node indices of the visit-level hypergraph, the representation from the \texttt{HConv} operator remains equivariant, i.e., $
\mathbf{Z}_{i,t}^{\mathrm{perm}} = \mathbf{P}\mathbf{Z}_{i,t}.
$
However, the output of the dynamic hypergraph SSM module remains invariant.
\end{theorem}

    The proof of this theorem is relegated to Appendix \ref{appendix: proof_perm_invariance}. It states that if node labels are permuted then only the representations from \texttt{Hconv} are relabeled whereas SSM output remains unchanged.

\section{Numerical Experiments}
In this section, we present results on the MIMIC III \citep{mimic3} and MIMIC IV \citep{mimic4} datasets  with downstream tasks as mortality prediction, length of stay (LOS) prediction, drug recommendation and readmission task. More details on datasets and tasks in Appendix \ref{appendix: dataset_traning_details}.

\begin{table}[t]

\begin{minipage}{0.49\linewidth}
\centering
\caption{Mortality Prediction.}
\setlength{\tabcolsep}{0.75pt}
\resizebox{\columnwidth}{!}{%
\begin{tabular}{lcccc}
\toprule
\textbf{Model} & \multicolumn{2}{c}{\textbf{MIMIC-III}} & \multicolumn{2}{c}{\textbf{MIMIC-IV}} \\
\cmidrule(lr){2-3} \cmidrule(lr){4-5}
 & AUPRC & AUROC & AUPRC & AUROC \\
\midrule
GRU        & 11.8$_{(0.5)}$ & 61.3$_{(0.9)}$ & 4.2$_{(0.1)}$ & 69.0$_{(0.8)}$ \\
Transformer& 10.1$_{(0.9)}$ & 57.2$_{(1.3)}$ & 4.0$_{(0.4)}$ & 65.1$_{(1.2)}$ \\
RETAIN     & 9.6$_{(0.6)}$  & 59.4$_{(1.5)}$ & 3.8$_{(0.4)}$ & 64.8$_{(1.6)}$ \\
GRAM       & 11.4$_{(0.7)}$ & 60.4$_{(0.9)}$ & 4.4$_{(0.3)}$ & 66.7$_{(0.7)}$ \\
Deepr      & 13.2$_{(1.1)}$ & 60.8$_{(0.4)}$ & 4.2$_{(0.2)}$ & 68.9$_{(0.9)}$ \\
StageNet   & 12.4$_{(0.3)}$ & 61.5$_{(0.7)}$ & 4.2$_{(0.3)}$ & 69.6$_{(0.8)}$ \\
GraphCARE  & 16.7$_{(0.5)}$ & 70.3$_{(0.5)}$ & 6.7$_{(0.3)}$ & 73.1$_{(0.5)}$ \\
\midrule
\texttt{HoT-SSM (v1)} & {\color{blue}\textbf{34.4}$_{(1.0)}$} & {\color{blue}\textbf{74.2}$_{(0.6)}$} & 23.8$_{(0.9)}$ & {\color{blue}\textbf{84.7}$_{(0.2)}$} \\
\texttt{HoT-SSM (v2)} & 32.8$_{(2.3)}$ & 73.2$_{(1.5)}$ & {\color{blue}\textbf{24.4}$_{(0.8)}$} & 82.8$_{(0.1)}$ \\
\bottomrule
\label{tab:mortality}
\end{tabular}
}
\end{minipage}
\hfill
\begin{minipage}{0.50\linewidth}
\centering
\caption{Drug Recommendation}
\setlength{\tabcolsep}{0.75pt}
\resizebox{\columnwidth}{!}{%
\begin{tabular}{lcccc}
\toprule
\textbf{Model} &  \multicolumn{2}{c}{\textbf{MIMIC-III}} & \multicolumn{2}{c}{\textbf{MIMIC-IV}} \\
\cmidrule(lr){2-3} \cmidrule(lr){4-5}
 & Jaccard & F1 & Jaccard & F1 \\
\midrule
GRU          & 47.8$_{(0.3)}$ & 60.2$_{(0.2)}$ & 44.0$_{(0.4)}$ & 60.2$_{(0.2)}$ \\
Transformer  & 47.1$_{(0.4)}$ & 55.9$_{(0.2)}$ & 40.4$_{(0.1)}$ & 55.9$_{(0.2)}$ \\
RETAIN       & 48.8$_{(0.2)}$ & 60.3$_{(0.1)}$ & 45.0$_{(0.1)}$ & 60.3$_{(0.1)}$ \\
GRAM         & 47.9$_{(0.3)}$ & 60.1$_{(0.2)}$ & 45.3$_{(0.3)}$ & 60.1$_{(0.1)}$ \\
Deepr        & 44.7$_{(0.3)}$ & 59.1$_{(0.4)}$ & 43.8$_{(0.4)}$ & 59.1$_{(0.4)}$ \\
StageNet     & 45.8$_{(0.4)}$ & 60.2$_{(0.3)}$ & 45.4$_{(0.4)}$ & 60.2$_{(0.3)}$ \\
GraphCARE    & 49.8$_{(0.4)}$ & 63.9$_{(0.3)}$ & 48.1$_{(0.3)}$ & 63.9$_{(0.3)}$ \\
\midrule
\texttt{HoT-SSM (v1)} & 52.2$_{(0.5)}$ & 66.5$_{(0.4)}$ & {50.6}$_{(0.6)}$ & {64.1}$_{(0.8)}$ \\
\texttt{HoT-SSM (v2)} & {\color{blue}\textbf{52.6}$_{(0.5)}$} & {\color{blue}\textbf{66.9}$_{(0.5)}$} & {\color{blue}\textbf{54.3}$_{(0.4)}$} & {\color{blue}\textbf{68.3}$_{(0.4)}$} \\
\bottomrule
\label{tab:drug_rec}
\end{tabular}
}
\end{minipage}
\end{table}

\textbf{Implementation details}. Visit-level hypergraphs for each patient are generated using Azure GPT-4.1, following the same procedure as described in Section~\ref{sec:hypergraph_generation} with prompts as in Fig.~\ref{fig:taxonomy_prompt} and Fig~\ref{fig:mapping_prompt}. These hypergraphs are then leveraged  to compute spatial representations via either the \texttt{HConv} or \texttt{HGAT} operators \citep{bai2021hypergraph}. Accordingly, we consider two variants of the proposed framework: \texttt{HoT-SSM} (v1), which employs \texttt{HConv} as the spatial encoder, and \texttt{HoT-SSM} (v2), which employs \texttt{HGAT}. The resulting spatial representations are propagated over time using the hypergraph-based state-space model (SSM) to produce the final predictions. Whereas the temporal reasoning paths are generated following the procedure in Sec~\ref{subsec: reasoning_paths}.  Details of hyperparameters and training configurations are provided in Appendix~\ref{appendix: training_details}.

\subsection{Discussion on Results}
In Tables~\ref{tab:mortality}, \ref{tab:drug_rec}, and \ref{tab:los} we present results (mean and standard deviation in brackets $(.)$) on three downstream clinical prediction tasks on the MIMIC-III and MIMIC-IV datasets. It can be observed that \texttt{HoT-SSM} demonstrates strong performance across all tasks compared to the baselines. In particular, it is important to notice that on mortality prediction, the proposed model achieves a improvement of \textbf{+17.7\% in AUPRC} and \textbf{+3.9\% in AUROC} over \texttt{GraphCARE} on the MIMIC-III dataset, asserting the importance of capturing higher-order relations for critical tasks. Whereas, on MIMIC-IV, we observe that the relative gains are much stronger with \textbf{+17.1\% in AUPRC} and \textbf{+11.6\% in AUROC} improvements  compared to MIMIC-III, as the dataset includes patients with more visits, which showcases the ability of the model to capture long-range patient history. It can also be observed that the proposed model outperforms state-of-the-art algorithms on other clinical tasks, asserting the importance of hypergraphs and SSMs.  Additional results on readmission prediction task are provided in Appendix~\ref{sec:readmission_task}. 

\subsection{Ablation Study 1: Effect of Hypergraph Modeling}
To study the impact of modeling higher-order clinical relationships with hypergraphs, we compare our proposed \texttt{HoT-SSM} framework against its graph-based counterpart, denoted as \texttt{GNN-SSM}. The latter models clinical relationships using knowledge graphs, similar to \texttt{GraphCare}~\citep{Graphcare}. Table~\ref{graph_ssm} reports the results on the mortality prediction task for both the MIMIC-III and MIMIC-IV datasets. As shown, \texttt{HoT-SSM} consistently outperforms \texttt{GNN-SSM}, demonstrating the advantage of explicitly modeling EHR data through hypergraphs.

\begin{table}[t]
\begin{minipage}[0.10\textheight]{0.55\linewidth}
\setlength{\tabcolsep}{1pt}
\caption{Performance on LOS prediction task.}
\begin{tabular}{lcccc}
\toprule
\multirow{2}{*}{\textbf{Model}} & \multicolumn{2}{c}{\textbf{MIMIC-III}} & \multicolumn{2}{c}{\textbf{MIMIC-IV}} \\
\cmidrule(lr){2-3} \cmidrule(lr){4-5}
 & \textbf{Kappa} & \textbf{F1} & \textbf{Kappa} & \textbf{F1} \\
\midrule
GRU         & 26.2$_{(0.2)}$ & 34.9$_{(0.5)}$ & 26.0$_{(0.1)}$ & 31.6$_{(0.2)}$ \\
Transformer & 25.4$_{(0.4)}$ & 34.8$_{(0.2)}$ & 25.3$_{(0.4)}$ & 31.4$_{(0.3)}$ \\
RETAIN      & 26.1$_{(0.4)}$ & 34.9$_{(0.4)}$ & 26.3$_{(0.2)}$ & 32.0$_{(0.2)}$ \\
GRAM        & 26.3$_{(0.3)}$ & 34.5$_{(0.2)}$ & 26.1$_{(0.4)}$ & 31.9$_{(0.3)}$ \\
Deepr       & 25.3$_{(0.4)}$ & 35.0$_{(0.4)}$ & 26.4$_{(0.2)}$ & 32.3$_{(0.1)}$ \\
StageNet    & 24.8$_{(0.3)}$ & 34.4$_{(0.4)}$ & 26.0$_{(0.2)}$ & 31.3$_{(0.3)}$ \\
GraphCARE   & 29.5$_{(0.4)}$ & 37.5$_{(0.2)}$ & {\color{blue}\textbf{29.8}$_{(0.3)}$} & {\color{blue}\textbf{34.2}$_{(0.3)}$} \\
\midrule
\texttt{HoT-SSM (v1)} & 30.2$_{(0.8)}$ & {37.7}$_{(0.8)}$ & 27.5$_{(0.3)}$ & {34.0}$_{(0.3)}$ \\
\texttt{HoT-SSM (v2)} & {\color{blue}\textbf{31.0}$_{(0.7)}$} & {\color{blue}\textbf{38.1}$_{(0.8)}$} & 28.8$_{(0.4)}$ & 33.5$_{(0.5)}$ \\
\bottomrule
\label{tab:los}
\end{tabular}
\end{minipage}
\end{table}
\hfill

\subsection{Ablation study 2: Effect of using SSM}
To examine the impact of long-range temporal information propagation on predictive performance, we conduct an ablation study on a subset of patients with extended clinical histories (those with more than 10 recorded visits) with the mortality prediction task on the MIMIC-IV dataset. This setting is particularly challenging, as accurate prediction requires preserving clinically relevant information over long temporal horizons. We compare \texttt{HoT-SSM} against two architectural variants with SSM module is replaced by (i)  Long Short-Term Memory (LSTM), (\texttt{HoT-LSTM}), and (ii) \texttt{Hyper-GNN}, which directly aggregates visit-level hypergraph representations without temporal modelling.  In Table~\ref{tab:ablation_long_range}, we report the results where it can be observed that \texttt{HoT-SSM} consistently outperforms both variants across all evaluation metrics. Notably, \texttt{HoT-SSM} achieves  larger gains with \textbf{+15\% in Sensitivity} and  \textbf{+6.5\% in Macro-F1 score} compared to the LSTM- baseline, thereby asserting the importance of proposed SSM framework for effectively capturing the long range clinical histories  especially in critical tasks such as mortality prediction.

\begin{figure}[b]
\centering
{
\begin{tcolorbox}
Mortality label=1, Predicted label=1 \\
\textbf{Visit 1:} \\
$\hookrightarrow$ \textcolor{blue}{\textbf{Entities:}} \textit{respiratory failure, insufficiency arrest} \\
\hspace{0.9cm}$\downarrow$ \\
\textbf{Visit 2:} \\
$\hookrightarrow$ \textcolor{blue}{\textbf{Entities:}} \textit{Respiratory intubation, mechanical ventilation} \\
\hspace{0.9cm}$\downarrow$ \\
\textbf{Visit 3:} \\
$\hookrightarrow$ \textcolor{blue}{\textbf{Entities:}} \textit{Other liver diseases, acute and unspecified renal failure} \\
\textbf{LLM explanation:} The model identifies this patient as high-risk due to a clinical trajectory consistent with multi-organ system failure. The sequence progresses from respiratory arrest requiring immediate resuscitation to mechanical ventilation for life support, culminating in liver disease at the final visit.
\end{tcolorbox}}
\caption{Gradient based temporal reasoning path.}
\label{fig:grad_reasoning_path}
\end{figure}
Further, to ensure that the observed gains are not solely due to LLM-based hyperedge construction, we investigate the following questions (see Appendix~\ref{sec:static_hypergraphs}).

\fcolorbox{black!50}{gray!10}{
\parbox{0.95\linewidth}{

\begin{itemize}
    \item \textbf{Q.1}. How does \texttt{HoT-SSM} perform with non-LLM based hypergraph constructions?
    \item \textbf{Q.2}.  How robust and reliable are the performance gains?
\end{itemize}
}} 
\subsection{Example Temporal Reasoning Path}
The temporal reasoning module in the \texttt{HoT-SSM} pipeline extracts the reasoning paths by gradient or attention based methods. In Fig.~\ref{fig:grad_reasoning_path}, we present the reasoning path generated by gradient based approach that selects the top-2 most influential clinical entities per visit based on gradient contribution for a representative mortality case from the MIMIC-III dataset. Grounding the reasoning abilities to extracted reasoning chains we also show an LLM output explaining the reasons for models prediction towards high mortality risk. More examples on temporal reasoning paths are presented  in Appendix~\ref{appendix: reasoning_path_examples}.
\begin{table}\begin{minipage}[0.2\textheight]{0.45\linewidth}
\vspace{0pt}
\setlength{\tabcolsep}{1pt}
\caption{Ablation on hypergraphs and graphs.}

\begin{tabular}{lcccc}
\toprule
\textbf{Model} & \multicolumn{2}{c}{\textbf{MIMIC-III}} & \multicolumn{2}{c}{\textbf{MIMIC-IV}} \\
\cmidrule(lr){2-3} \cmidrule(lr){4-5}
 & AUPRC & AUROC & AUPRC & AUROC \\
\midrule
\texttt{GNN-SSM}  & 22.9 & 69.5 & 19.8 & 77.4 \\
\rowcolor{red!15}\texttt{HoT-SSM} & \textbf{34.4} & \textbf{74.2} & \textbf{23.7} & \textbf{84.6} \\
\bottomrule
\label{graph_ssm}
\end{tabular}

\vspace{8pt}
\caption{Ablation on long-range task.}

\begin{tabular}{lccc}
\toprule
\textbf{Model} & \textbf{AUPRC} & \textbf{Macro F1} & \textbf{Sensitivity} \\
\midrule
Hyper-GNN & 60.6 & 62.8 & 25.0 \\
HoT-LSTM     & 65.1 & 68.9 & 40.0 \\
\rowcolor{red!15}\texttt{HoT-SSM} & \textbf{67.1} & \textbf{75.4} & \textbf{55.0} \\
\bottomrule
\label{tab:ablation_long_range}
\end{tabular}

\end{minipage}
\end{table}

\section{Conclusions}

We proposed \texttt{HoT-SSM}, a novel higher-order temporal reasoning framework that integrates hypergraph modeling with state-space models for healthcare. The proposed framework constructs knowledge infused temporal hypergraphs to encode co-occurring clinical concepts within visits and leverages a novel hypergraph-based SSM to capture long-range dependencies across patient visits. We also extract reasoning paths based on attention and gradient-based approaches, allowing LLM-assisted explanation for the interpretable inference. Results on MIMIC-III and MIMIC-IV demonstrate that \texttt{HoT-SSM} effectively captures higher-order relations and long-range temporal information, leading to substantial gains in the predictive performance.

\bibliographystyle{unsrtnat}
\bibliography{ref}

\newpage
\onecolumn

\appendix

\section*{\centering APPENDIX} 


\noindent\makebox[\linewidth]{\rule{\textwidth}{0.4pt}}

\begin{enumerate}
    \item Prior Works \dotfill \pageref{appendix: prior_works}
    \item Proof of Theorems \dotfill \pageref{appendix: Proof of Theorems}
    \item Implementation Details \dotfill \pageref{appendix: implementation_details}
    \item Dataset and Training Details \dotfill \pageref{appendix: dataset_traning_details}
    \item Additional Experiments \dotfill \pageref{appendix: additional_experiments}
    \item Reasoning Path Examples \dotfill \pageref{appendix: reasoning_path_examples}
    \item Future Directions \dotfill \pageref{appendix: future_directions}
\end{enumerate}
\section{Prior Works}
\label{appendix: prior_works}
Prior research on electronic health record (EHR) data for downstream clinical prediction and recommendation tasks can be broadly categorized into representation learning-based approaches and  methods centered on KGs and LLMs.
\paragraph{Representation Learning–Based Approaches.}
Early work in this category applied classical machine learning models and deep learning architectures to learn low-dimensional representations of patient data, which were subsequently used for downstream clinical tasks \citep{ wong2018using, mcwilliams2019towards, choi2016doctor, choi2016retain, zhang2021grasp, stagenet, ashrafi2024optimizing, chung2014empirical}. While effective, these methods largely overlook the relational graph structure underlying the data.

To explicitly model such relationships, graph neural network (GNN)-based approaches have been proposed, where clinical entities are represented as nodes and their interactions as edges \citep{choi2020learning, su2020gate, li2020graph, zhu2021variationally, yang2023molerec, xie2019ehr}. These models construct patient-specific or visit-level graphs from EHR data and leverage message passing to learn relational representations. However, most representation learning-based approaches operate on local graphs derived solely from EHR records and do not incorporate external medical knowledge sources, such as PubMed or UMLS, which contain rich and complementary relational information.  Recent works leverage hypergraphs \citep{hyperehr,hypkg} to model EHR data. However they model each visit as a single hyperedge and multiple visits through static hypergraph, thereby failing to explicitly capture intra-visit structure (whose importance is discussed in the Introduction and Section~\ref{sec:hypergraph_generation}) as well as inter-visit temporal dynamics. This limitation restricts their ability to effectively model disease progression. More importantly, these approaches lack an explicit mechanism to capture long-range patient information over time.

\paragraph{Approaches centered on Knowledge Graphs and LLMs}
Personalized medical knowledge graphs (MKGs) provide a structured representation of EHR data by explicitly modeling semantic relationships among clinical entities. Early approaches in this direction constructed MKGs using predefined hierarchical or rule-based schemas and demonstrated notable performance improvements on downstream clinical tasks \citep{ping2017individualized,gyrard2018personalized,shirai2021applying}. Despite their effectiveness, these methods are limited in their ability to capture complex and higher-order clinical interactions.

With recent advances in large language models (LLMs) pretrained on extensive medical corpora, more recent work has explored the construction of personalized knowledge graphs using LLMs, often augmented with external biomedical resources such as UMLS and PubMed \citep{Graphcare,KARE,gao2025leveraging}. While these approaches enrich the semantic coverage of MKGs, they typically restrict interactions to pairwise relations. As a result, they fail to capture co-occurring clinical concepts within a single visit and lack explicit mechanisms for modeling temporal evolution or preserving long-range clinical dependencies across a patient’s history. 


\section{Proof for Theorems}
\label{appendix: Proof of Theorems}
We present the proofs for robustness to pertubations and permutation invariance.
\subsection{Proof of Theorem~\ref{thm:perturb}}
\label{appendix: proof_pertub}
\begin{proof}
    Let the normalized hypergraph adjacency operator be defined as
\[
\mathbf{E}_{i,t}
=
\mathbf{Dv}_{i,t}^{-1/2}
\mathbf{I}_{i,t}
\mathbf{W}
\mathbf{De}_{i,t}^{-1}
\mathbf{I}_{i,t}^{\top},
\]
 where $\mathbf{I}_{i,t}$ denotes the incidence matrix, and $\mathbf{Dv}_{i,t}$ and $\mathbf{De}_{i,t}$ are the node and hyperedge degree matrices, respectively.   
    The spatial representations using \texttt{Hconv} can be represented interms of normalized hypergraph Laplacian as
    \begin{equation}
       \mathbf{Z}^{l+1}_{i,t} = \sigma\left((\mathbf{I} - \mathbf{L})\mathbf{Z}^{l}_{i,t}\mathbf{\Theta}^{(l+1)}\right), 
    \end{equation}
    where $\mathbf{L}$ is a normalized hypergraph Laplacian. Under the perturbation of hypergraph Laplacian, the error in the representations are bounded as
    \begin{equation}
        \begin{aligned}
        \|\hat{\mathbf{Z}}^{l+1}_{i,t}- \mathbf{Z}^{l+1}_{i,t}\|_{2} &{=} \|\sigma\left((I-\hat{\mathbf{L}})\mathbf{Z}^{(l)}_{i,t}\mathbf{\Theta}^{(l+1)}\right)-\sigma\left((I-\mathbf{L})\mathbf{Z}^{(l)}_{i,t}\mathbf{\Theta}^{(l+1)}\right)\|_{2} \\
           &\overset{(a)}{\leq} C ||\Delta\mathbf{L}\mathbf{Z}^{(l)}_{i,t}\mathbf{\Theta}^{(l+1)}\|_{2}\\
            &\overset{(b)}{\leq}K C \|\Delta\mathbf{L}\|_{2} \|\mathbf{Z}^{(l)}_{i,t}\|_{2} \\
            &\overset{(c)} {\leq} \epsilon KC \|\mathbf{Z}^{(l)}_{i,t}\|_{2}, \\
            &\overset{(d)} {\leq} \epsilon LKC \|\mathbf{X}_{i,t}\|_{2},
        \end{aligned}
        \label{eq:per_1}
    \end{equation}
  where \eqref{eq:per_1}(a) follows  since $\sigma(.)$ being Lipschitz continous and \eqref{eq:per_1}(b) follows from the bound on the learnable parameters $\Theta$. Whereas \eqref{eq:per_1}(c) follows from the bound on the energy of the pertubation matrix and \eqref{eq:per_1}(d) follows by recursively applying the above bound across $L$ \texttt{HConv} layers and
noting that $\mathbf{Z}^{(0)}_{i,t}=\mathbf{X}_{i,t}$.

 Recall $\tilde{\mathbf{x}}$ is obtained using visit node embedding and the error can be bounded interms of representations of \texttt{Hconv} layer as  
\begin{equation}
        \begin{aligned}
\|\tilde{\mathbf{x}}_{i,t,per}- \tilde{\mathbf{x}}_{i,t}\|_{2} &\overset{(a)}{=} \|\mathbf{e}^{T}_{v}\hat{\mathbf{Z}}_{i,t} \mathbf{B}^{T} - \mathbf{e}^{T}_{v}{\mathbf{Z}}_{i,t} \mathbf{B}^{T}\|_{2} \\
&\overset{(b)}{\leq} \epsilon LKC \|\mathbf{X}_{i,t}\|_{2} \gamma,
        \end{aligned}
        \label{eq:per_2}
    \end{equation}
where \eqref{eq:per_2}(a) follows by extracting the visit node embedding using canonical vector and \eqref{eq:per_2}(b) follows from \eqref{eq:per_1}(d) with $\|\mathbf{B}\|_{2}\leq\gamma$.

Finally the error between the ssm representations are bounded as 

\begin{equation}
    \begin{aligned}
        \|\hat{\mathbf{s}}_{i,k}-\mathbf{s}_{i,k}\|_{2} &\overset{(a)}= \|\tilde{\mathbf{x}}_{i,t,per}\rm{exp}(\Delta_{k}\mathbf{A}-\mathbf{I})\mathbf{A}^{-1}-\tilde{\mathbf{x}}_{i,t}\rm{exp}(\Delta_{k}\mathbf{A}-\mathbf{I})\mathbf{A}^{-1}\|_{2}\\
        &\overset{(b)}\leq \rm{const} \cdot \epsilon \|\mathbf{X}_{i,t}\|_{2}, 
    \end{aligned}
    \label{eq:per_3}
\end{equation}
where $\rm{const} = LKC\gamma\delta$ in \eqref{eq:per_3}(b), follows from   \eqref{eq:per_2}(a).  We emphasize that from \eqref{eq:per_3}(b),  error in temporal representations from \texttt{HoT-SSM} scales only linearly with the energy in the perturbation matrix demonstrating the model robustness to graph perturbations.  
\end{proof}
\subsection{Proof of Corollary ~\ref{sec:cor}} \label{sec:proof_cor}
We extend the analysis by explicitly unrolling the SSM dynamics. Recall that the difference between perturbed and unperturbed states satisfies:
\begin{equation} \begin{aligned} \|\hat{\mathbf{s}}_{i,1}-\mathbf{s}_{i,1}\|_{2} &\overset{(a)}= \|(\hat{\mathbf{s}}_{i,0}-\mathbf{s}_{i,0})\rm{exp}(\Delta_{1}\mathbf{A})\|_{2}+\|(\tilde{\mathbf{x}}_{i,1,per}-\tilde{\mathbf{x}}_{i,1})\rm{exp}(\Delta_{1}\mathbf{A}-\mathbf{I})\mathbf{A}^{-1}\|_{2}\\ &\overset{(b)}\leq \eta \epsilon L K C \delta \gamma \|\mathbf{X}_{i,1}\|_{2}, \end{aligned} \end{equation}
where (b) is realized by leveraging eqn~\ref{eq:per_3} and by  observing that initializations remain the same. At the second step, the error consists of (i) propagated error from the previous step and (ii) newly injected perturbation:
\begin{equation} \begin{aligned} \|\hat{\mathbf{s}}_{i,2}-\mathbf{s}_{i,2}\|_{2} &\overset{(a)}= \|(\hat{\mathbf{s}}_{i,1}-\mathbf{s}_{i,1})\rm{exp}(\Delta_{2}\mathbf{A})\|_{2}+\|(\tilde{\mathbf{x}}_{i,2,per}-\tilde{\mathbf{x}}_{i,2})\rm{exp}(\Delta_{2}\mathbf{A}-\mathbf{I}))\mathbf{A}^{-1}\|_{2}\\ &\overset{(b)}\leq \eta \epsilon L K C \delta \gamma \|\mathbf{X}_{i,1}\|_{2} + \epsilon L K C \delta \gamma \|\mathbf{X}_{i,2}\|_{2} \end{aligned}  \end{equation}
where (b) is realized through $\|\exp(\Delta_k \mathbf{A})\|_2 \le \eta$.

Similarly, for $k=3$ we have 
\begin{equation} \begin{aligned} \|\hat{\mathbf{s}}_{i,3}-\mathbf{s}_{i,3}\|_{2} &\leq \eta^{2} \epsilon L K C \delta \gamma \|\mathbf{X}_{i,1}\|_{2}+\eta \epsilon L K C \delta \gamma \|\mathbf{X}_{i,2}\|_{2} + \epsilon L K C \delta \gamma \|\mathbf{X}_{i,3}\|_{2} \end{aligned} \label{eq:pper_3} 
\nonumber
\end{equation}

Generalizing this for any $k=K$, we have
\begin{equation} \begin{aligned} \|\hat{\mathbf{s}}_{i,K}-\mathbf{s}_{i,K}\|_{2} &\leq \epsilon L K C \delta \gamma \sum_{\tau=1}^{K} \eta^{K-\tau}\|\mathbf{X}_{i,\tau}\|_{2} \\ &\overset{(b)}\leq  \rm{const}_{1}. \epsilon \end{aligned} \label{eq:pper_4} \end{equation}

Here $\rm{const}_{1} = L K C \delta \gamma \sum_{\tau=1}^{K} \eta^{K-\tau}\|\mathbf{X}_{i,\tau}\|_{2}$.
This expression explicitly captures the temporal residual propagation, where perturbations introduced at earlier time steps are attenuated (or propagated) through the factor $\eta^{K-\tau}$.

Importantly, the above result shows that the error remains linear in the perturbation magnitude $\epsilon$, and
the proportionality constant now explicitly captures both layerwise and temporal cascading effects.

\subsection{Perturbation Analysis Interms of the Spectrum}\label{sec:proof_spectrum}
To establish the inequality in terms of spectrum of hypergraph, we follow the  similar lines as done earlier. In particular, recall perturbed hypergraph Laplacian is given by $\hat{\mathbf{L}} = \mathbf{L} + \Delta\mathbf{L}$ with  $ \mathbf{L}$  and $\hat{\mathbf{L}}$ being normalized. Therefore, their respective eigenvalues $\lambda_{i}$ and $\hat{\lambda}_{i}$ satisfy
$
\lambda_i \in [0,2], \quad \hat{\lambda}_i \in [0,2].
$
Then the corresponding spectral decompositions be given by
$
\mathbf{L} = \mathbf{U}\mathbf{\Lambda}\mathbf{U}^\top, \quad
\hat{\mathbf{L}} = \hat{\mathbf{U}}\hat{\mathbf{\Lambda}}\hat{\mathbf{U}}^\top,
$
where $\mathbf{U}$ and $\hat{\mathbf{U}}$ are orthonormal matrices. Note that in general the eigenvectors are not preserved under perturbations.

The error between the representations are now expressed as 
\begin{equation}
        \begin{aligned}
        \|\hat{\mathbf{Z}}^{l+1}_{i,t}- \mathbf{Z}^{l+1}_{i,t}\|_{2} &{=} \|\sigma\left((I-\hat{\mathbf{L}})\mathbf{Z}^{(l)}_{i,t}\mathbf{\Theta}^{(l+1)}\right)- \\&  \sigma\left((I-\mathbf{L})\mathbf{Z}^{(l)}_{i,t}\mathbf{\Theta}^{(l+1)}\right)\|_{2} \\
           &\overset{(a)}{\leq} C ||(\mathbf{L}-\hat{\mathbf{L}})\mathbf{Z}^{(l)}_{i,t}\mathbf{\Theta}^{(l+1)}\|_{2}\\
            &\overset{(b)}{\leq}K C (\|\mathbf{L}\|_{2} + \|\hat{\mathbf{L}}\|_{2}\|)\|\mathbf{\Theta}^{(l+1)}\|_{2}\|\mathbf{Z}^{(l)}_{i,t}\|_{2} \\
            &\overset{(c)} {\leq}  LKC (\lambda_{\max}+ \hat{\lambda}_{\max})\|\mathbf{X}_{i,t}\|_{2},
        \end{aligned}
        \label{eq:pper_1}
    \end{equation}
where (a) follows from Lipschitz continuity, (b) follows from the triangle inequality and (c) follows since Laplacian is symmetric and their spectral norms equal their largest eigenvalues:
$
\|\mathbf{L}\|_2 = \lambda_{\max}, \quad
\|\hat{\mathbf{L}}\|_2 = \hat{\lambda}_{\max}.
$

Therefore, the errors in the representations are bounded by the maximum eigenvalues of true and perturbed graph hypergraph Laplacians that has  maximum eigenvalue as 2.

Further, if we impose structure on the perturbation-specifically modeling it as arising from hyperedge addition or deletion-the perturbed hypergraph Laplacian can be written in the form
$
\hat{\mathbf{L}} = \mathbf{L} + \rho \sum_{j} \mathbf{b}_{j} \mathbf{b}_{j}^{\top},
$
where $ \rho \in \{+1,-1\}$ corresponds to edge addition or deletion, respectively, and $\mathbf{b}_j \in \mathbb{R}^{N \times 1}$ is an incidence vector with entries as $+1$ at one end point node and $-1$ at other end point node, and zero elsewhere.  denotes the canonical basis vectors. Under such structured perturbations, existing results from graph signal processing \citep{ceci2020graph} show that the spectrum of the perturbed  Laplacian can be related to that of the original graph Laplacian. In particular, the largest eigenvalue $\hat{\lambda}_{\max}$ can be characterized in terms of $\lambda_{\max}$.

Therefore, while the proposed theorem is formulated in terms of the perturbation energy $( \|\Delta \mathbf{L}\|_2 )$, the imposed structure on the perturbation enables an equivalent interpretation in terms of the spectra of the normal and perturbed Laplacians. This provides additional insight into how structural changes in the hypergraph affect the learned representations.
\subsection{Proof of Theorem~\ref{thm:permutation_invariance}}
\label{appendix: proof_perm_invariance}
\begin{proof}
Recall the normalized hypergraph adjacency operator is defined as
\[
\mathbf{E}_{i,t}
=
\mathbf{Dv}_{i,t}^{-1/2}
\mathbf{I}_{i,t}
\mathbf{W}
\mathbf{De}_{i,t}^{-1}
\mathbf{I}_{i,t}^{\top},
\]
where $\mathbf{I}_{i,t}$ denotes the incidence matrix, and $\mathbf{Dv}_{i,t}$ and $\mathbf{De}_{i,t}$ are the node and hyperedge degree matrices, respectively.

Under a permutation of node labels induced by $\mathbf{P}$, the adjacency matrix modifies as
\[
\mathbf{E}_{i,t}^{\mathrm{perm}} = \mathbf{P}\mathbf{E}_{i,t}\mathbf{P}^{\top}.
\]
Similarly, the node feature matrix is permuted as
$\mathbf{X}_{i,t}^{\mathrm{perm}} = \mathbf{P}\mathbf{X}_{i,t}$.

The hypergraph convolution layer computes node representations as
\[
\mathbf{Z}_{i,t}^{(l+1)}
=
\sigma\left(
\mathbf{E}_{i,t}\mathbf{Z}_{i,t}^{(l)}\mathbf{\Theta}^{(l+1)}
\right).
\]
Under the permuted node labels and inputs, representations modifies as
\begin{align}
\mathbf{Z}_{i,t,\mathrm{perm}}^{(l+1)}
&\overset{(a)}{=}
\sigma\left(
\mathbf{E}_{i,t}^{\mathrm{perm}}
\mathbf{Z}_{i,t,\mathrm{perm}}^{(l)}
\mathbf{\Theta}^{(l+1)}
\right) \nonumber \\
&\overset{(b)}{=}
\sigma\left(
\mathbf{P}\mathbf{E}_{i,t}\mathbf{P}^{\top}
\mathbf{P}\mathbf{Z}_{i,t}^{(l)}
\mathbf{\Theta}^{(l+1)}
\right) \nonumber \\
&\overset{(c)}{=}
\mathbf{P}
\sigma\left(
\mathbf{E}_{i,t}\mathbf{Z}_{i,t}^{(l)}
\mathbf{\Theta}^{(l+1)}
\right) \nonumber \\
&\overset{(d)}=
\mathbf{P}\mathbf{Z}_{i,t}^{(l+1)},
\label{eq:permutation_equivariance}
\end{align}
where the \eqref{eq:permutation_equivariance}(c) follows from the orthogonality of permutation matrices and the fact that $\sigma(\cdot)$ is applied row-wise. \eqref{eq:permutation_equivariance}(d) establishes that representations from \texttt{HConv} are permutation equivariant.

Recall, the input to the dynamic state-space model is the embedding of the visit node, extracted as
\[
\mathbf{z}^{v}_{i,t} = \mathbf{e}_{v}^{\top}\mathbf{Z}_{i,t},
\]
where $\mathbf{e}_{v} \in \mathbb{R}^{|\mathcal{V}_{i,t}|}$ is the canonical basis vector corresponding to the visit node.
Under permutation, the node representations are reordered as in~\eqref{eq:permutation_equivariance}, yielding
\begin{align}
\mathbf{z}^{v}_{i,t,\rm{perm}}&=
\mathbf{e}_{v}^{\top}\mathbf{P}^{T}\mathbf{P}\mathbf{Z}_{i,t}. \nonumber \\
&\overset{(a)}{=}\mathbf{e}_{v}^{\top}\mathbf{Z}_{i,t}, \nonumber\\
&=\mathbf{z}^{v}_{i,t}
\end{align}
where (a) follows from the property of  permutation matrix. Since the proposed model operates solely on the visit-level representations  although \texttt{Hconv} introduces the permutation equivariance at the node level, visit node selection induces invariance hence the output from \texttt{Hot-SSM } are invariant to permutations. 
\end{proof}

\section{Implementation Details}
\label{appendix: implementation_details}
\subsection{Hyperknowledge Graph Construction}
\label{appendix: hyperknowledge_graph}

\begin{figure*}[t]
    \includegraphics[width=\textwidth]{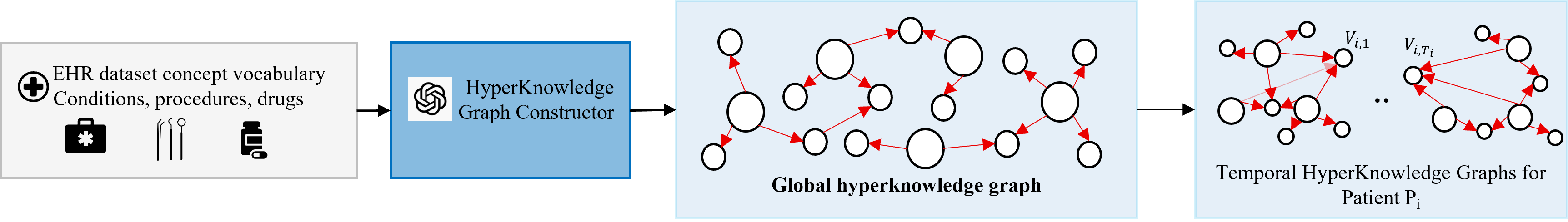}
    
    \caption{Our pipeline to construct hyperknowledge Graphs}
    \label{fig:hyperknowledge_graph_construction}
   
\end{figure*}
The construction of the hyperknowledge graph is performed in two sequential stages, as illustrated in Figure~\ref{fig:hyperknowledge_graph_construction}.

\textbf{Global HyperKnowledge Graph Construction}:
To ensure semantic consistency across the EHR dataset, we create global hyperknowledge graph in two phases inspired by TnT-LLM architecture ~\citep{wan2024tnt}. 
\begin{itemize}
    
    \item Taxonomy  (Phase 1): We first induce a global medical taxonomy by processing a representative subset of the corpus. This stage corresponds to the taxonomy generation phase in TnT-LLM, where a global clinical taxonomy is constructed. Using a diverse sample, we prompt the LLM to identify high-level clinical concepts that serve as global hyperedges or buckets (e.g., Heart Failure). These hyperedges constitute the foundation of the global hyperknowledge graph. The prompt used for clinical bucket generation is shown in Figure~\ref{fig:taxonomy_prompt}, and example hyperedges are presented in Table~\ref{tab:hyperedge_examples}.

    \item Guided Batch  (Phase 2): In this phase, we scale taxonomy assignment across full dataset while maintaining semantic consistency. Medical entities are processed in batches, where LLM is conditioned on cumulative set of existing hyperedges discovered in Phase 1. The LLM is instructed to reuse existing hyperedges whenever a semantic match exists and is allowed to introduce new hyperedges only when an entitiy represents fundamentally distinct clinical concept. Furthermore, recognizing the nature of clinical concepts, our prompting supports multi-label assignment (polysemy), this allows a single entity to map to multiple relevant hyperedges (e.g., an entity "Topical products for joint and muscular pain" is mapped to both "Bone and Joint Disease" and "Pain Management"). The prompt for this phase is shown in Figure ~\ref{fig:mapping_prompt} and examples are shown in Table~\ref{tab:entity_mapping_examples}.
\end{itemize}
\begin{table}[t]
\centering
\caption{Examples of global hyperedges constructed during the taxonomy discovery phase.}
\label{tab:hyperedge_examples}

\begin{tabularx}{\textwidth}{@{}X@{\hspace{1.5em}}X@{}}
\toprule
\multicolumn{2}{c}{\textbf{Example Global Hyperedges}} \\
\midrule

\textbf{Gastrointestinal Disease} &
\textbf{Heart Failure} \\

\cmidrule(lr){1-1}\cmidrule(lr){2-2}

\textit{Example entities:} &
\textit{Example entities:} \\

\begin{minipage}[t]{\linewidth}
\begin{itemize}[leftmargin=*, nosep, topsep=2pt]
    \item Agents for treatment of hemorhoids and anal fissures (topical)
    \item Antacids
    \item Antidiarrheal microorganisms
    \item Gastritis and duodenitis
    \item Nausea and vomiting
    \item Gastrointestinal hemorrhage
\end{itemize}
\end{minipage}
&
\begin{minipage}[t]{\linewidth}
\begin{itemize}[leftmargin=*, nosep, topsep=2pt]
    \item Congestive heart failure; nonhypertensive
    \item Diagnostic ultrasound of heart (echocardiogram)
    \item Swan-ganz catheterization for monitoring
    \item Beta blocking agents
    \item Cardiac glycosides
    \item Potassium-sparing diuretics
\end{itemize}
\end{minipage}
\\
\bottomrule
\end{tabularx}
\end{table}

\begin{figure*}[!t]
\centering
\begin{tcolorbox}[colback=white, colframe=black, width=\textwidth]

\textbf{Role:} You are a Senior Medical Ontologist and Knowledge Graph Architect working with ICU data. You have a dataset of \textit{[Total Entities]} ICU medical entities (Conditions, Procedures, Drugs).\\

\textbf{Your Task:} Analyze the diverse random sample below and define a comprehensive set of \textbf{Medical Fact Buckets} that:
\begin{enumerate}
    \item Cover these sample items
    \item Would likely cover other similar ICU medical concepts in the full dataset
    \item Are broad enough to group many entities but distinct enough to be meaningful
\end{enumerate}

\textbf{Core Principle:}
\begin{itemize}
    \item A bucket = a medical fact
    \item Everything clinically related to that fact belongs in it, whether it is a condition, a procedure, or a drug
\end{itemize}

\textbf{Instructions:}
\begin{enumerate}
    \item The bucket name must be the medical fact itself
    \item Create a reasonable number based on the clinical diversity you observe in the sample:
    \begin{itemize}
        \item Not so few that unrelated concepts get merged together
        \item Not so many that you recreate a flat list of individual entities
    \end{itemize}
    \item One entity can belong to \textbf{multiple buckets} if clinically relevant
    \begin{itemize}
        \item Example: ``metoprolol'' in [``Hypertension'', ``Heart Failure'', ``Atrial Fibrillation'']
    \end{itemize}
    \item Use exact entity names
    \item A bucket should focus on clinical relationships (what treats what, what manages what)
\end{enumerate}

\textbf{Sample Entities (\textit{[Batch Size]} items):}
\begin{itemize}[leftmargin=*, noitemsep, topsep=0pt]
    \item \textbf{Conditions:} \textit{[Comma-separated list of conditions...]}
    \item \textbf{Procedures:} \textit{[Comma-separated list of procedures...]}
    \item \textbf{Drugs:} \textit{[Comma-separated list of drugs...]}
\end{itemize}

\textbf{Output STRICT JSON format:}
\begin{lstlisting}
{
  "suggested_buckets": [
    "Heart Failure"
  ]
}
\end{lstlisting}

\textbf{CRITICAL:} Return ONLY the JSON object. Bucket names must be concise medical facts/concepts.

\end{tcolorbox}

\caption{The prompt used for Taxonomy Discovery (Phase I), where the LLM induces clinically grounded medical fact buckets from a representative scout batch of ICU entities.}
\label{fig:taxonomy_prompt}

\end{figure*}


\begin{figure*}[!t]
\centering
\begin{tcolorbox}[colback=white, colframe=black, width=\textwidth]

\textbf{Role:} You are a medical knowledge graph expert.

\textbf{Task:} Map the following \textit{[Batch Size]} entities to one or more \textbf{Medical Concept Buckets}.

\textbf{Core Principle:}
\begin{itemize}
    \item A bucket = a medical fact/concept
    \item Conditions, procedures, and drugs all belong in the same bucket if they are clinically associated with that fact
\end{itemize}

\textbf{Consistency Constraint:}
\begin{quote}
    \textit{[Insert consistency instruction / existing buckets here]}
\end{quote}

\textbf{Rules:}
\begin{enumerate}[leftmargin=*, noitemsep, topsep=0pt]
    \item \textbf{REUSE:} If an entity fits into one of the EXISTING buckets above, you MUST use that exact name.
    \item \textbf{CREATE:} Only create a NEW bucket name if the entity represents a distinctly new clinical concept not covered above.
    \item \textbf{Polysemy:} An entity can belong to MULTIPLE buckets if it has distinct clinical uses.
    \begin{itemize}
        \item Example: ``beta blocking agents'' $\rightarrow$ [``Management of Hypertension'', ``Management of Heart Failure'']
    \end{itemize}
    \item Bucket names must be concise medical concept nouns.
    \item \textbf{Exhaustiveness:} You MUST output a key-value pair for EVERY SINGLE entity in the input list below.
\end{enumerate}

\textbf{Input Entities:}
\begin{lstlisting}
[
  "Aspirin",
  "Metoprolol",
  "Chest X-Ray",
  ... (list of entities in current batch)
]
\end{lstlisting}

\textbf{Output STRICT JSON format:}
\begin{lstlisting}[breaklines=true]
{
  "entity_mappings": {
    "exact_entity_name_from_input": ["Bucket Name 1", "Bucket Name 2"],
    "another_entity_name": ["Bucket Name 3"]
  }
}
\end{lstlisting}

\textbf{Important:}
\begin{itemize}
    \item The keys of the JSON object MUST be the exact strings from the input list - no typo fixes, no case changes.
    \item You MUST output a key-value pair for EVERY single entity in the input list.
    \item Count the input entities (\textit{[Batch Size]}) and ensure the output JSON has the same number of keys.
    \item Values must be arrays/lists of bucket names, even if only one bucket.
    \item A bucket should focus on clinical relationships (what treats what, what manages what)
\end{itemize}

\textbf{CRITICAL:} Return ONLY the JSON object.

\end{tcolorbox}

\caption{The updated prompt used for Phase II (Guided Batch Mapping), enforcing reuse of existing buckets and clinically grounded concept mapping.}
\label{fig:mapping_prompt}

\end{figure*}



\begin{table*}[b]
\centering
\caption{Examples of multi-label entity-to-hyperedge mappings from the guided batch assignment phase.}
\label{tab:entity_mapping_examples}

\begin{tabularx}{\textwidth}{
    >{\raggedright\arraybackslash}p{0.42\textwidth}
    >{\raggedright\arraybackslash}X
}
\toprule
\textbf{Medical Entity} & \textbf{Mapped Global Hyperedges} \\
\midrule

\begin{minipage}[c]{\linewidth}
Topical products for joint and muscular pain
\end{minipage}
&
\begin{minipage}[c]{\linewidth}
\begin{itemize}[leftmargin=*, noitemsep, topsep=0pt]
    \item Bone and Joint Disease
    \item Pain Management
\end{itemize}
\end{minipage}
\\

\midrule

\begin{minipage}[c]{\linewidth}
Selective calcium channel blockers with direct cardiac effects
\end{minipage}
&
\begin{minipage}[c]{\linewidth}
\begin{itemize}[leftmargin=*, noitemsep, topsep=0pt]
    \item Hypertension
    \item Cardiac Arrhythmias
    \item Coronary Artery Disease
\end{itemize}
\end{minipage}
\\

\bottomrule
\end{tabularx}
\end{table*}

\subsubsection{Knowledge Infused Temporal Hypergraph Construction}

Given a patient visit, we construct knowledge infused temporal hypergraphs by grounding the global hyperknowledge graph to the medical concepts observed in that visit.
We first identify all medical concepts present in the visit and select only those hyperedges from the global hyperknowledge graph that contain at least one of these concepts.
For each selected hyperedge, we retain only the concepts that actually appear in the visit.

In addition, we introduce a visit node and connect it to all selected hyperedges to explicitly model visit-level context.
The resulting visit-level hypergraph is defined as
\begin{equation}
\mathcal{H}_{i,t} = (\mathcal{V}_{i,t}, \mathcal{E}_{i,t}),
\end{equation}
where $\mathcal{V}_{i,t}$ denotes the set of unique clinical entities appearing in the visit $\mathcal{C}_{i,t}$, together with the visit node, and
\begin{equation}
\mathcal{E}_{i,t} = \{ e_{i,t}^{(1)}, e_{i,t}^{(2)}, \ldots, e_{i,t}^{(K_{i,t})} \}
\end{equation}
represents the set of hyperedges associated with the visit.
A single visit may be connected to one or multiple hyperedges, depending on the number of clinically related concept groups. Table \ref{tab:hypergraph_stats} represents statistics for temporal hyperknowledge graphs.

\begin{table}[htbp]
\centering
\caption{Statistics for MIMIC-III and MIMIC-IV Hypergraph Representations}
\label{tab:hypergraph_stats}
\resizebox{\columnwidth}{!}{%
\begin{tabular}{llrrrrrr}
\toprule
\textbf{Dataset} & \textbf{Task} & \textbf{Patients} & \textbf{Visits} & \textbf{Total} & \multicolumn{3}{c}{\textbf{Hyperedges per Visit}} \\
\cmidrule(lr){6-8}
 &  &  &  & \textbf{Hyperedges} & \textbf{Mean (SD)} & \textbf{Min} & \textbf{Max} \\
\midrule
\multirow{4}{*}{\textbf{MIMIC-III}} 
 & Mortality Prediction & 6,186 & 9,717 & 263736 & 27.14 (6.53) & 4 & 33 \\
 & Readmission Prediction & 6,186 & 9,717 & 263736 & 27.14 (6.53) & 4 & 33 \\
 & Length of Stay Prediction & 35,707 & 44,399 & 1,383,419 & 31.16 (11.07) & 3 & 69 \\
 & Drug Recommendation & 35,707 & 44,399 & 1,204,642 & 27.13 (9.17) & 3 & 57 \\
\midrule
\multirow{4}{*}{\textbf{MIMIC-IV}} 
 & Mortality Prediction & 59,262 & 132,275 & 3,277,198 & 24.78 (8.21) & 3 & 54 \\
 & Readmission Prediction & 59,262 & 132,275 & 2,698,999 & 20.40 (6.67) & 2 & 55 \\
 & Length of Stay Prediction & 123,478 & 232,247 & 6,711,737 & 28.90 (12.03) & 2 & 93 \\
 & Drug Recommendation & 46,184 & 154,953 & 3,281,934 & 21.18 (6.37) & 2 & 44 \\
\bottomrule
\end{tabular}
}
\end{table}

\subsection{Attention based Hyperknowledge Graph Convolution}
\label{appendix: attention_based_hykg}
The attention based approach assigns learnable importance weights to clinical concepts and hyperedges, enabling explicit identification of the factors that contribute to the model’s prediction.
By learning these weights, the model can emphasize clinically relevant concepts and relations, thereby improving interpretability.

Let $v_r \in \mathcal{V}_{i,t}$ denote a clinical concept node r associated with patient $i$ at visit $t$.
The attention based convolution is applied at each network layer and operates in two stages:
(1) aggregation from nodes to hyperedges, followed by
(2) aggregation from hyperedges to nodes.

\textbf{Node to hyperedge aggregation.}
For a hyperedge $e^j \in \mathcal{E}_{i,t}$, its representation at layer $l+1$ is computed by attending over the nodes incident to it:
\begin{equation}
\hat{\mathbf{e}}^{(l+1)}_{j}
= \sigma\left(
\sum_{v_r \in e^j}
\alpha_{j,r}
\mathbf{W}^{(1)} \mathbf{z}^{(l)}_{r}
\right),
\end{equation}
where $\mathbf{z}^{(l)}_{r} \in \mathbb{R}^{d \times 1}$ denotes the representation of node $v_r$ at layer $l$,
$\mathbf{W}^{(1)} \in \mathbb{R}^{d \times d}$ is a learnable projection matrix, and $\sigma(\cdot)$ denotes a non linear activation function.

The node to hyperedge attention coefficient $\alpha_{j,r}$ reflects the importance of node $v_r$ within hyperedge $e^j$ and is defined as
\begin{equation}
\alpha_{j,r}
=
\frac{
\exp\left(\mathbf{m}^{\top} \tilde{\mathbf{z}}_{r}\right)
}{
\sum_{v_q \in e^j}
\exp\left(\mathbf{m}^{\top} \tilde{\mathbf{z}}_{q}\right)
},
\end{equation}
where $\mathbf{m} \in \mathbb{R}^{d}$ is a learnable attention vector and
\begin{equation}
\tilde{\mathbf{z}}_{r}
= \text{LeakyReLU}\left(\mathbf{W}^{(1)} \mathbf{z}^{(l)}_{r}\right).
\end{equation}

\textbf{Hyperedge to node aggregation.}
The updated hyperedge representations are then propagated back to nodes through an attention mechanism.
The representation of node $v_r$ at layer $l+1$ is computed as
\begin{equation}
\mathbf{z}^{(l+1)}_{r}
= \sigma\left(
\sum_{e^j \in \mathcal{E}_{i,t}(v^r)}
\beta_{r,j}
\mathbf{W}^{(2)} \hat{\mathbf{e}}^{(l+1)}_{j}
\right),
\nonumber
\end{equation}
where $\mathcal{E}_{i,t}(v^r)$ denotes the set of hyperedges incident to node $v^r$, $e^{j}$ denotes $j^{th}$ hyperedge and $\mathbf{W}^{(2)} \in \mathbb{R}^{d \times d}$ is a learnable transformation.  $\hat{\mathbf{e}}^{(l+1)}_{j}$ is the updated representation of hyperedge $e^j$.

The hyperedge to node attention coefficient $\beta_{r,j}$ quantifies the influence of hyperedge $e^j$ on node $v_r$ and is defined as
\begin{equation}
\beta_{r,j}
= \frac{
\exp\left(\mathbf{n}^{\top} \tilde{\mathbf{h}}_{i,j}\right)
}{
\sum_{e^q \in \mathcal{E}_{i,t}(v^r)}
\exp\left(\mathbf{n}^{\top} \tilde{\mathbf{h}}_{q,j}\right)
},
\end{equation}
where $\mathbf{n} \in \mathbb{R}^{d}$ is a learnable attention vector and
\begin{equation}
\tilde{\mathbf{h}}_{i,j}
=
\text{LeakyReLU}\left(
\mathbf{W}^{(2)} \hat{\mathbf{e}}^{(l+1)}_{j}
\;\Vert\;
\mathbf{W}^{(1)} \mathbf{z}^{(l)}_{i}
\right).
\end{equation}

These learned attention weights $\beta_{r,j}$ are used in generating reasoning chains.

\subsection{Learnable Parameters}
Complementing the parameter efficiency analysis in Figure~\ref{fig:pred_size}, here we present the  learnable parameters in \texttt{HoT-SSM} framework.  These only includes input projection matrices for dimensionality alignment and the weight matrices inherent to \texttt{HConv} layers. Then in the SSM module the only learnable parameters are the continuous input matrix $\mathbf{B}$ and the output matrix $\mathbf{C}$. Since the state transition matrix $\mathbf{A}$ is HiPPO initialized (not learnable). Finally decoder involves learnable  projection layer.
\subsection{Computation Overhead of Generating Hyperedges}
Recall, the proposed method avoids naive per-patient hypergraph construction by employing a two-phase design optimized for scalability and reduced cost. Specifically, a global hyperknowledge graph is constructed once as an offline preprocessing step. In particular, global hyperknowledge graph construction for MIMIC-III, incurs a runtime as 178 sec and token cost as 100k (total tokens) 
$\times$ per token cost, while the per-patient hypergraph construction incurs runtime as 0.2 milliseconds. Importantly, this global construction is not repeated during training or inference, making it efficient in practice. Furthermore, when deploying in a new setting with expanded vocabulary, the system only processes the newly observed terms against the existing taxonomy, ensuring that additional token and computational costs remain controlled.

\section{Dataset and Training Details}
\label{appendix: dataset_traning_details}
\begin{table}[t]
\centering
\caption{Statistics of EHR datasets. ``\#'': ``the number of'', ``/patient'': ``per patient''.}
\label{tab:dataset_statistics}
\resizebox{\columnwidth}{!}{
\begin{tabular}{lcccccc}
\hline
Dataset & \#patients & \#visits & \#visits/patient & \#conditions/patient & \#procedures/patient & \#drugs/patient \\
\hline
MIMIC-III & 35,707 & 44,399 & 1.24 & 12.89 & 4.54 & 33.71 \\
MIMIC-IV  & 123,488 & 232,263 & 1.88 & 21.74 & 4.70 & 43.89 \\
\hline
\end{tabular}
}
\end{table}

\subsection{Dataset Description}
We use publicly available MIMIC-III \citep{mimic3} and MIMIC-IV \citep{mimic4} datasets. Table \ref{tab:dataset_statistics} summarizes the key statistics of datasets used in our experiments.
\subsection{Task Description}
\label{app:prediction_tasks}

We evaluate our model on four tasks using the electronic health record (EHR) data. Each task is formulated based on a patient’s visit sequence. 

\begin{itemize}
    \item \textbf{Mortality Prediction.}
    This binary classification task  predicts patient survival status during the next visit.
    The model uses a patient’s previous visit records to output a binary prediction indicating survival status.
    During dataset preparation the final visit of each patient is excluded from training since it does not have a future outcome.

    \item \textbf{Readmission Prediction.}
    This binary classification task determines whether a patient will be readmitted to the hospital within 15 days after discharge. The model predicts a binary label based on the time interval between consecutive hospital visits

    \item \textbf{Length of Stay Prediction.}
    This task estimates the duration of a patient’s intensive care unit (ICU) stay for a given hospital visit.
    It is formulated as a multi-class classification problem with 10 discrete categories. The target label is represented as a one-hot vector indicating the corresponding class among 10 categories, which represent ICU stays of less than one day (class 0), within one week (classes 1-7), between one and two weeks (class 8), and longer than two weeks (class 9).

    \item \textbf{Drug Recommendation.}
    This task predicts the set of medications prescribed to a patient during a hospital visit.
    The model outputs a set of drugs selected from a predefined medication vocabulary.
    Since multiple drugs may be prescribed simultaneously, this task is treated as a multi-label classification problem.
\end{itemize}

\textbf{Dataset Preprocessing}: Both MIMIC-III and MIMIC-IV datasets are split into train, validation and test splits with 80\%/10\%/10\% ratio. For dataset preprocessing we followed the similar setup as in \citep{Graphcare,KARE}.

\textbf{Baselines}: We compare the proposed model against representation learning approaches. In particular we compare against GRU \citep{chung2014empirical}, Transformer \citep{vaswani2017attention},
RETAIN \citep{choi2016retain}, GRAM \citep{choi2017gram}, Deepr \citep{deepr}, StageNet \citep{stagenet}, AdaCare \citep{ma2020concare}, GRASP \citep{zhang2021grasp} and GraphCare \citep{Graphcare}.

\textbf{Evaluation Metrics:} Model performance is evaluated using a comprehensive set of metrics designed to capture distinct aspects of predictive reliability. \textbf{AUROC} (Area Under the Receiver Operating Characteristic Curve), \textbf{AUPRC} (Area Under the Precision-Recall Curve), \textbf{F1 Score}, \textbf{Accuracy}, \textbf{Jaccard}, and \textbf{Cohen's Kappa}.

We use binary cross-entropy loss for mortality, readmission. Cross-entropy loss for drug recommendation and length of stay prediction.

\subsection{Training Details}
\label{appendix: training_details}

We preprocess the EHR datasets following the same methodology as in GraphCare \citep{Graphcare}. 
 The dataset is partitioned into  80\%/10\%/10\% for training/validation/testing data and we use Adam as the optimizer, patience as 85, runs as 10 on MIMIC-III and 5 on MIMIC IV. The hyperparameters used for training are summarized in Table \ref{tab:hotssm_hparams}, Table \ref{tab:hotssm_hparams_v2}.

\begin{table}[htbp]
\centering
\caption{Hyperparameter configuration for \texttt{HoT-SSM (v1)} across tasks on MIMIC-III  and MIMIC-IV. Mort: Mortality. Red: Readmission. Drug: Drug Recommendation}
\label{tab:hotssm_hparams}
\begin{tabular}{lcccccccc}
\toprule
\multirow{2}{*}{\textbf{Hyperparameter}} & \multicolumn{4}{c}{\textbf{MIMIC-III}} & \multicolumn{4}{c}{\textbf{MIMIC-IV}} \\
\cmidrule(lr){2-5} \cmidrule(lr){6-9}
& \textbf{Drug} & \textbf{LOS} & \textbf{Mort.} & \textbf{Read.} & \textbf{Drug} & \textbf{LOS} & \textbf{Mort.} & \textbf{Read.} \\
\midrule
Step size $\Delta$          & 0.001 & 0.001 & 0.001 & 0.001 & 0.001 & 0.001 & 0.001 & 0.001 \\
Learning rate ($lr$)          & 0.001 & 0.001 & 0.001 & 0.001 & 0.001 & 0.001 & 0.001 & 0.001 \\
Hidden dimension              & 384   & 256   & 128   & 128   & 384   & 128   & 128   & 128 \\
Convolution layers            & 3     & 3     & 2     & 2     & 2     & 2     & 2     & 2 \\
SSM state dimension           & 384    & 128   & 128   & 128   & 64    & 128   & 128   & 128 \\
Batch size                    & 128   & 128   & 128   & 128   & 64   & 128   & 128   & 128 \\
Dropout                       & 0.1   & 0.0   & 0.2   & 0.1   & 0.2   & 0.1   & 0.3   & 0.3 \\
\bottomrule
\end{tabular}
\end{table}

\begin{table}[htbp]
\centering
\caption{Hyperparameter configuration for \texttt{HoT-SSM (v2)} across tasks on MIMIC-III  and MIMIC-IV.}
\label{tab:hotssm_hparams_v2}
\begin{tabular}{lcccccccc}
\toprule
\multirow{2}{*}{\textbf{Hyperparameter}} & \multicolumn{4}{c}{\textbf{MIMIC-III}} & \multicolumn{4}{c}{\textbf{MIMIC-IV}} \\
\cmidrule(lr){2-5} \cmidrule(lr){6-9}
& \textbf{Drug} & \textbf{LOS} & \textbf{Mort.} & \textbf{Read.} & \textbf{Drug} & \textbf{LOS} & \textbf{Mort.} & \textbf{Read.} \\
\midrule
Stepsize ($\Delta$)         & 0.001 & 0.001 & 0.001   & 0.001 & 0.001 & 0.001 & 0.001 & 0.001 \\
Learning rate ($lr$)          & 0.001 & 0.001 & 0.001 & 0.001 & 0.001 & 0.001 & 0.001 & 0.001 \\
Hidden dimension              & 384   & 256   & 128   & 128   & 256   & 128   & 256   & 128 \\
Convolution layers            & 2     & 2     & 2     & 2     & 2     & 2     & 2     & 2 \\
SSM state dimension           & 384   & 128   & 128   & 128   & 128   & 128   & 128   & 128 \\
Batch size                    & 64    & 128    & 128   & 128   & 64    & 128   & 128   & 128 \\
Attention heads               & 2    & 2    & 2   & 1   & 2    & 2    & 2    & 1 \\
Dropout                       & 0.1   & 0.1   & 0.1   & 0.1   & 0.2   & 0.3   & 0.3   & 0.1 \\
\bottomrule
\end{tabular}
\end{table}

\subsection{System Configuration}
All experiments were conducted on a server equipped with two 64-core Intel Xeon Platinum 8562Y+ CPUs with 512GB memory. The platform runs on Ubuntu 22.04.5 LTS with GCC version 10.5.0. We used CUDA 11.8, Pytorch version
2.1.2 and Pytorch-geometric 2.7.0 for all the experiments. All experiments were performed on a single NVIDIA A40 GPU with 44GB of VRAM.

\section{Additional Experiments}
\label{appendix: additional_experiments}

\subsection{Ablation based on source of hyperknowledge creation}\label{sec:static_hypergraphs}

To evaluate the robustness of our framework to the source of domain knowledge, we include additional experiments where hyperedges are constructed without any LLM involvement. Specifically, we consider (i) ontology-based grouping using standard medical ontologies. The individual medical entities are mapped to broad, standardized category labels (e.g., "Circulatory System Diseases") using the CCS and ATC ontologies to build global hyperedges. Then, for each patient visit, a specific hypergraph is constructed by activating only the predefined hyperedges corresponding to the entities present in that patient's visits, formally linking the concepts and a central visit node to these hyperedges via an incidence matrix and  
(ii) co-occurrence-based grouping, where concept co-occurrence frequencies are first computed across all patient visits and normalized using Positive Pointwise Mutual Information (PPMI). The resulting concept relationships are then clustered into $50$ groups for drug recommendation task and $100$ groups for all other task, each forming a hyperedge, creating global hyperknowledge graph. For each patient visit, a specific hypergraph is subsequently instantiated by activating only the hyperedges corresponding to concepts present in that record.
As shown in Table~\ref{tab:llm_ablation_mortality_readmission}, Table~\ref{tab:llm_ablation_los_drug}, HoT-SSM with these non-LLM constructions achieves performance that is comparable to HoT-SSM where hyperedges are built using LLMs. Further, we emphasize that performance of this approach is still significantly better than, prior state-of-the-art methods. While LLM-based hyperedges provide additional semantic refinement, the \emph{core gains persist even with static or data-driven constructions, indicating that the improvements primarily stem from modeling higher-order relations via hypergraphs and long-range temporal dependencies via SSM}. Overall, these results demonstrate that HoT-SSM is robust to the choice of hyperedge construction and remains effective even in resource-constrained settings without LLMs.


\begin{table}[t]
\centering
\caption{Comparison of hyperedge construction strategies on MIMIC-III for mortality and readmission prediction task.}
\resizebox{\columnwidth}{!}{
\begin{tabular}{llcccc}
\hline
\textbf{Category} & \textbf{Method} 
& \multicolumn{2}{c}{\textbf{Mortality}} 
& \multicolumn{2}{c}{\textbf{Readmission}} \\
 &  & \textbf{AUPRC} & \textbf{AUROC} 
 & \textbf{AUPRC} & \textbf{AUROC} \\
\hline
\multirow{2}{*}{Without LLM} 
& Unsupervised clustering 
& $31.33 \pm 0.5$ & $71.71 \pm 0.8$ 
& $65.00 \pm 0.4$ & $62.51 \pm 0.3$ \\
& Medical ontologies 
& $31.98 \pm 1.6$ & $71.84 \pm 0.7$ 
& $64.45 \pm 0.5$ & $62.25 \pm 0.6$ \\
\hline
\rowcolor{red!15}
With LLM
& HoT-SSM 
& $34.40 \pm 1.0$ & $74.27 \pm 0.6$ 
& $67.45 \pm 0.9$ & $64.59 \pm 0.6$ \\
\hline
\end{tabular}
}
\label{tab:llm_ablation_mortality_readmission}
\end{table}

\begin{table}[t]
\centering
\caption{Comparison of hyperedge construction strategies on MIMIC-III for LOS and Drug Recommendation task.}
\resizebox{\columnwidth}{!}{
\begin{tabular}{llcccc}
\hline
\textbf{Category} & \textbf{Method} 
& \multicolumn{2}{c}{\textbf{LOS}} 
& \multicolumn{2}{c}{\textbf{Drug}} \\
 &  & \textbf{Kappa} & \textbf{F1} 
 & \textbf{Jaccard} & \textbf{F1} \\
\hline
\multirow{2}{*}{Without LLM} 
& Unsupervised clustering 
& $27.54 \pm 0.3$ & $35.20 \pm 0.4$ 
& $40.71 \pm 1.3$ & $55.17 \pm 1.5$  \\
& Medical ontologies 
& $27.12 \pm 0.1$ & $35.52 \pm 0.2$ 
& $47.12 \pm 0.7$ & $61.34 \pm 0.2$  \\
\hline
\rowcolor{red!15}
With LLM
& HoT-SSM 
& $30.20 \pm 0.8$ & $37.44 \pm 0.8$ 
& $52.67 \pm 0.5$  & $66.98 \pm 0.5$ \\
\hline
\end{tabular}
}
\label{tab:llm_ablation_los_drug}
\end{table}

\subsubsection{Calibration metrics on  mortality prediction}

To evaluate the robustness of performance gains, we report calibration metrics for mortality prediction task on MIMIC-III in Table~\ref{tab:calibration_metrics}. It can be observed that \texttt{HoT-SSM} achieves lower ECE and Brier score, along with higher specificity and sensitivity, indicating more reliable and well-calibrated predictions. We emphasize compared to prior state-of-the-art methods, \texttt{HoT-SSM} demonstrates \emph{improved robustness, further supporting our claims}.
\begin{table}[h]
\centering
\caption{Calibration metrics}
\begin{tabular}{lcccc}
\hline
\textbf{Method} & \textbf{ECE} & \textbf{Brier Score} & \textbf{Sensitivity} & \textbf{Specificity} \\
\hline

\texttt{GraphCare} 
& -- 
& -- 
& $0.17$ 
& $0.97$ \\
\texttt{KARE} 
& -- 
& -- 
& $0.14$ 
& $0.94$ \\
\rowcolor{red!15}
\texttt{HoT-SSM} 
& $0.03$ 
& $0.10$ 
& $0.18$ 
& $0.98$ \\
\hline
\end{tabular}
\label{tab:calibration_metrics}
\end{table}

\subsection{Comparison with KARE}
As shown in Table \ref{tab:comparision_with_kare}, HoT-SSM consistently outperforms KARE in both zero-shot and few-shot settings for mortality and readmission prediction on MIMIC-III. While the fine-tuned variant of KARE achieves the highest Macro-F1, it depends on task-specific LLM fine-tuning and careful hyperparameter optimization. More importantly, fine tuned variant incurs significant computational and token costs. In contrast, HoT-SSM attains strong performance without any LLM fine-tuning, highlighting its effectiveness on EHR prediction tasks.
\begin{table}[htbp]
\centering
\caption{Macro F1 (\%) comparison of  \texttt{HoT-SSM} and \texttt{KARE} on MIMIC-III for mortality and readmission prediction.}
\label{tab:comparision_with_kare}
\begin{tabular}{lcc}
\toprule
\textbf{Model} & \textbf{Mortality} & \textbf{Readmission} \\
\midrule
\texttt{KARE} (Zero-shot)      & 54.6 & 56.3 \\
\texttt{KARE} (Few-shot)       & 53.5 & 57.1 \\
\texttt{KARE} (Fine-tuned)     & 64.6 & 73.7 \\
\midrule
\texttt{HoT-SSM} (ours) & 59.1 & 60.4 \\
\bottomrule
\end{tabular}
\end{table}

\subsection{Results on Readmission task} \label{sec:readmission_task}
Table~\ref{tab:readmission_results} shows the performance comparison between \texttt{Hot-SSM} and other baselines on readmission task. It can be observed that proposed model achieves competitive performance against state-of-the-art baselines. We also emphasize that \texttt{HoT-SSM} attains competitive performance on this particular task with a minimal parameter complexity as shown Fig.~\ref{fig:pred_size}.

\begin{table*}[htbp]
\centering
\caption{Performance comparison on readmission prediction task}
\label{tab:readmission_results}
\begin{tabular}{lcccc}
\toprule
 & \multicolumn{4}{c}{\textbf{Task: Readmission Prediction}} \\
\cmidrule(lr){2-5}
\textbf{Model} 
& \multicolumn{2}{c}{\textbf{MIMIC-III}} 
& \multicolumn{2}{c}{\textbf{MIMIC-IV}} \\
\cmidrule(lr){2-3} \cmidrule(lr){4-5}
 & AUPRC & AUROC & AUPRC & AUROC \\
\midrule
GRU         & 68.2$_{(0.4)}$ & 65.4$_{(0.8)}$ & 66.1$_{(0.7)}$ & 66.2$_{(0.1)}$ \\
Transformer & 67.3$_{(0.7)}$ & 63.9$_{(1.1)}$ & 65.7$_{(0.3)}$ & 65.3$_{(0.4)}$ \\
RETAIN      & 65.1$_{(1.0)}$ & 64.1$_{(0.7)}$ & 66.2$_{(0.3)}$ & 65.3$_{(0.2)}$ \\
GRAM        & 67.2$_{(0.8)}$ & 64.3$_{(0.4)}$ & 66.1$_{(0.2)}$ & 66.3$_{(0.3)}$ \\
Deepr       & 68.8$_{(0.9)}$ & 66.5$_{(0.4)}$ & 65.6$_{(0.1)}$ & 65.4$_{(0.2)}$ \\
AdaCare     & 68.6$_{(0.6)}$ & 65.7$_{(0.3)}$ & 65.9$_{(0.0)}$ & 66.1$_{(0.0)}$ \\
GRASP       & 69.2$_{(0.4)}$ & 66.3$_{(0.6)}$ & 66.3$_{(0.3)}$ & 66.1$_{(0.2)}$ \\
StageNet    & 69.3$_{(0.6)}$ & 66.7$_{(0.4)}$ & 66.1$_{(0.1)}$ & 66.2$_{(0.1)}$ \\
GraphCARE   & \textbf{73.4}$_{(0.4)}$ & \textbf{69.7}$_{(0.5)}$ & \textbf{69.6}$_{(0.3)}$ & \textbf{68.5}$_{(0.4)}$ \\
\midrule
 {\texttt{HoT-SSM (v1)}} & 67.4$_{(0.9)}$ & 64.5$_{(0.6)}$ & 66.3$_{(0.2)}$ & 66.2$_{(0.4)}$ \\
{\texttt{HoT-SSM (v2)}} & 66.1$_{(1.9)}$ & 63.2$_{(1.8)}$ & 66.1$_{(1.9)}$ & 63.2$_{(1.8)}$ \\
\bottomrule
\end{tabular}
\end{table*}
\subsection{Results on Phenotyping task}
In Table~\ref{tab:phenotyping task}, we report the results on the phenotyping task to enable a fair comparison with prior methods, namely \texttt{HypEHR} \citep{hyperehr} and \texttt{HypKG}\citep{hypkg}, which are commonly evaluated on this benchmark. The results show that the proposed model consistently outperforms existing approaches, highlighting the importance of jointly modeling higher-order relationships. More importantly, it underscores the effectiveness of proposed method for generating  hyperknowledge graphs and capturing long-range dependencies for improved performance.
\begin{table}[htbp]
\caption{Performance comparison on phenotyping task.}
\centering
\begin{tabular}{lcc}
\toprule
\textbf{Method} & \textbf{AUROC} & \textbf{F1 Score} \\
\midrule
HypEHR  & $82.19 \pm 0.13$ & $41.51 \pm 0.52$ \\
HypKG   & $84.26 \pm 0.17$ & $45.30 \pm 0.49$ \\
\rowcolor{red!10} HoT-SSM & $85.95 \pm 0.01$ & $47.87 \pm 0.01$ \\
\bottomrule
\end{tabular}
\label{tab:phenotyping task}
\end{table}

\section{Reasoning Path Examples}
\label{appendix: reasoning_path_examples}
 
 In this section, we present additional examples of reasoning paths obtained using both attention-based and gradient-based approaches. The attention-based method extracts reasoning paths from the learned attention coefficients $\beta$ and $\alpha$ where higher values of $\beta$ indicate more influential hyperedges, higher $\alpha$ indicate influential entities and relevant to the model’s prediction. As illustrated in Fig.~\ref{fig:reasoning_path_A}, attention-based explanations capture richer clinical context by reflecting the model’s message-passing decisions and preserving temporal and structural coherence as each hyperedge directly encodes cooccuring concepts (condition, procedures, drugs) thereby giving more context to LLM.  In contrast, while gradient-based approaches are parameter-efficient, they primarily measure sensitivity rather than causal contribution, often resulting in less coherent reasoning paths.
 
 The prompt used for generating the explanation is given in figure~\ref{fig:prompt_template}.

\begin{figure}[t]
    \centering
    \begin{tcolorbox}
        \small
        Based on the following reasoning paths generated for patient mortality prediction, provide a detailed textual explanation of what happened:
        
        \vspace{0.5em}
        \textbf{Reasoning Paths Data:} \\
        \texttt{\{formatted\_output\}}
        
        \vspace{0.5em}
        Please provide a clear, structured explanation of the reasoning paths, including:
        \begin{enumerate}
            \itemsep0em 
            \item What the paths represent
            \item Key entities and relationships identified
            \item How they relate to mortality prediction
            \item Any notable patterns or insights
            \item Give concise text explanation
        \end{enumerate}
    \end{tcolorbox}
    \caption{The prompt template used to generate textual explanations from the reasoning paths. The \texttt{\{formatted\_output\}} placeholder is replaced by the specific paths extracted from the TKG.}
    \label{fig:prompt_template}
\end{figure}

\section{Limitations}\label{sec:limitations}
The proposed framework depends on large language models (LLMs) for constructing hyperknowledge graphs, which can be susceptible to hallucinations and inconsistencies. Although the model demonstrates strong performance even with static hypergraph construction methods, leveraging LLMs provides additional accuracy gains. However, this benefit comes at the cost of increased token usage and associated computational overhead.

\section{Future Directions}
\label{appendix: future_directions}
While this work focuses on healthcare applications, the proposed framework is general and can be extended to domains that requires long-context reasoning over structured data, such as question answering and retrieval-augmented generation, where exploiting higher-order relational structure with SSMs is a promising direction.

\begin{figure}[H]
    \centering
    {
    \begin{tcolorbox}
    \textbf{[I]} \\
    Mortality label=1, Predicted label=1 \\
\textbf{Visit 1:} \\
$\hookrightarrow$ \textcolor{blue}{\textbf{Gastrointestinal Disease}:} \textit{Biliary Tract Disease, Drugs for Peptic Ulcer and GORD} \\
$\hookrightarrow$ \textcolor{blue}{\textbf{Genitourinary Disease}:} \textit{Urinary Tract Infections, Other Diseases of Bladder and Urethra} \\

\vspace{0.1cm}
\textbf{Visit 2:} \\
$\hookrightarrow$ \textcolor{blue}{\textbf{Gastrointestinal Disease}:} \textit{Drugs for Peptic Ulcer and GORD, Drugs for Constipation} \\
$\hookrightarrow$ \textcolor{blue}{\textbf{Cancer}:} \textit{Cancer of Uterus, Other Non-Epithelial Cancer of Skin} \\

\vspace{0.1cm}
\textbf{Visit 3:} \\
$\hookrightarrow$ \textcolor{blue}{\textbf{Cancer}:} \textit{Cancer of Uterus} \\
$\hookrightarrow$ \textcolor{blue}{\textbf{Other and Ill-Defined Heart Disease}:} \textit{Peri-, Endo-, and Myocarditis Cardiomyopathy} \\

\vspace{0.1cm}
\textbf{Visit 4:} \\
$\hookrightarrow$ \textcolor{blue}{\textbf{Infectious Disease}:} \textit{Immunostimulants, Mycoses} \\
$\hookrightarrow$ \textcolor{blue}{\textbf{Cancer}:} \textit{Immunostimulants, Cancer of Uterus} \\

\vspace{0.2cm}
\textbf{LLM explanation}: The patient’s trajectory demonstrates a highly complex clinical course dominated by uterine cancer, compounded by recurrent gastrointestinal and genitourinary issues. The subsequent development of multiple infections and progressive vascular/hypertensive disorders, requiring continuous antibacterial and cardiovascular interventions, indicates a chronic and compounding decline in systemic health, leading to a high mortality risk.
    \tcblower
    \textbf{[II]} \\
     Mortality label=1, Predicted label=1 \\
    \textbf{Visit 1:} \\
    $\hookrightarrow$ \textcolor{blue}{\textbf{Entities:}} \textit{Anxiolytics and Biliary Tract Disease} 
    
    \vspace{0.1cm}
    \hspace{0.5cm}$\downarrow$ 
    \vspace{0.1cm}
    
    \textbf{Visit 2:} \\
    $\hookrightarrow$ \textcolor{blue}{\textbf{Entities:}} \textit{Cancer Of Uterus and Other Antibacterials, } 
    
    \vspace{0.1cm}
    \hspace{0.5cm}$\downarrow$ 
    \vspace{0.1cm}
    
    \textbf{Visit 3:} \\
    $\hookrightarrow$ \textcolor{blue}{\textbf{Entities:}} \textit{Cancer Of Uterus and Peri-, Endo-, And Myocarditis Cardiomyopathy}
     \vspace{0.1cm}

     \vspace{0.1cm}
    \hspace{0.5cm}$\downarrow$ 
    \vspace{0.1cm}

     \textbf{Visit 4:} \\
    $\hookrightarrow$ \textcolor{blue}{\textbf{Entities:}} \textit{Immunostimulants and Complication Of Device Implant Or Graft}
     \vspace{0.1cm}

     \textbf{LLM explanation}: The patient’s trajectory begins with biliary tract disease and the use of anxiolytics, rapidly escalating with a major diagnosis of uterine cancer. The clinical course is further complicated by infections requiring antibacterials and the development of severe cardiovascular issues, specifically peri-, endo-, and myocarditis cardiomyopathy. By the final visit, the administration of immunostimulants and the emergence of complications from a device implant or graft highlight a highly complex, progressive systemic decline driven by both the primary malignancy and compounding therapeutic burdens.
    
    \end{tcolorbox}
    }
    \caption{For a mortality prediction of patient  in MIMIC-III, [I] important hyperedges(blue) and corresponding top entities identified using attention based reasoning method, [II] important entities and reasoning path identified by gradient based method}
    \label{fig:reasoning_path_A}
\end{figure}

\end{document}